\documentclass{article} 
\usepackage{iclr2027_conference,times}

\usepackage{amsmath,amsfonts,bm}

\def\eqref#1{equation~\ref{#1}}

\def\1{\bm{1}}

\DeclareMathAlphabet{\mathsfit}{\encodingdefault}{\sfdefault}{m}{sl}
\SetMathAlphabet{\mathsfit}{bold}{\encodingdefault}{\sfdefault}{bx}{n}

\newcommand{\parent}{\ensuremath{\text{parent}}}
\newcommand{\sft}{\ensuremath{\text{SFT}}}
\newcommand{\sftc}{\ensuremath{\text{SAC}}}

\newcommand{\prompt}{\ensuremath{I}}
\newcommand{\err}{\ensuremath{\mathcal E}}

\usepackage{url}

\usepackage{tcolorbox}
\tcbuselibrary{breakable}
\usepackage{booktabs}
\usepackage{tabularx}
\usepackage{enumitem}
\usepackage[table]{xcolor}

\usepackage{amsthm}
\newtheorem{theorem}{Theorem}

\newtheorem{assumption}{Assumption}
\renewcommand{\Pr}{\operatorname{Pr}}

\usepackage{tocloft}
\usepackage{etoc}

\usepackage{hyperref}
\usepackage{cleveref}
\usepackage{xurl}

\title{SFT-as-Context Mitigates Forgetting \\ in Supervised Fine-Tuning}

\author{Kenan Tang\thanks{Equal contribution}, Andong Hua\footnotemark[1], Chengxuan Qian, Saket Tiwari, Yao Qin \\
University of California, Santa Barbara\\
\texttt{\{kenantang, dongx1997, chengxuanqian, saket, yaoqin\}@ucsb.edu}
}

\iclrfinalcopy 
\begin{document}

\maketitle

\begin{abstract}
Supervised fine-tuning (SFT) equips large language models (LLMs) with specialized capabilities, but often comes at the cost of forgetting the general capabilities of their parent models (i.e., the pretrained models before fine-tuning).
This trade-off is especially limiting for queries that require both specialized and general capabilities.
We introduce SFT-as-context, a training-free method in which the parent model uses the SFT model's response as context to answer the query.
This allows the parent model to acquire fine-tuned capabilities from the SFT response through in-context learning while preserving its own general capabilities.
Across \textbf{19 parent-SFT model pairs and 11 benchmarks}, SFT-as-context remains close to the SFT models on fine-tuned capabilities, with gaps of only 2.2 and 2.1 percentage points on AIME 2024 and LiveCodeBench and 2.0 macro MAE on NutriBench-English, while staying within 2.2 percentage points of the parent models on general capabilities on average.
Remarkably, it can solve queries requiring both fine-tuned and general capabilities, even when neither the parent nor SFT model succeeds alone. This approach also extends beyond parent-SFT pairs: responses from a small open-source SFT model can improve a strong \emph{closed-source} LLM, outperforming either model alone.
Furthermore, we use a Bayesian framework to derive theoretical guarantees that bound the error of SFT-as-context relative to the SFT model on fine-tuned capabilities and to the parent model on general capabilities.
In addition, we visualize the attention weights and find that the parent model attends more to useful SFT responses and less to irrelevant ones, suggesting that selective attention helps the parent model use the SFT response through in-context learning.

\end{abstract}

\section{Introduction}
Supervised fine-tuning (SFT) is widely used to adapt general-purpose large language models (LLMs) to specialized tasks and domains for practical deployment~\citep{chung2024scaling, zeng-etal-2024-agenttuning, guo2025deepseek}.
However, SFT models suffer from \textit{\textbf{forgetting}}, where fine-tuned models lose general capabilities that their parent models (i.e., the pretrained models before fine-tuning) already possess~\citep{kirkpatrick2017overcoming, liu-etal-2024-catastrophic, vu-etal-2022-overcoming}.
For example, \Cref{fig:capability_demonstration}(a) shows that a model fine-tuned for nutrition estimation on English meal-description queries achieves higher accuracy than the parent model, but it follows the fine-tuning output format and thus incorrectly assigns a carbohydrate value even when the input is a non-food query such as ``Hello, world!'', as shown in \Cref{fig:capability_demonstration}(b). In contrast, the parent model correctly recognizes that the input is unrelated to food or nutrition. This strongly indicates that the SFT model forgets a general language-understanding capability, which is how to respond appropriately to non-food queries.

This forgetting problem suffered by SFT models becomes particularly severe when a single query requires both the fine-tuned capabilities gained through fine-tuning and the general capabilities. For example, as shown in \Cref{fig:capability_demonstration}(c), when the meal-description query is in a language different from the fine-tuned language, the model must accurately estimate nutrition using its fine-tuned capability while responding in the same language as the input using its general multilingual capability, so that the user can understand the response. The SFT model significantly improves accuracy over its parent but, unlike the parent model, fails to respond in the user's language. This practical setting demonstrates that neither the SFT model nor the parent model alone satisfies both requirements.

Most existing approaches seek to mitigate forgetting through changes to training, including regularization~\citep{li2017learning,chen2020recall}, replay~\citep{sun2020lamol,he2024seekr}, constraints on parameter updates~\citep{lopez2017gradient,ICLR2026_4e447acb}, and the choice of training method~\citep{chen2026retaining,biderman2024lora}.
However, these approaches require modifying the fine-tuning process and thus cannot be directly applied, without further training, to existing ready-to-use checkpoints on platforms such as Hugging Face.\footnote{\url{https://huggingface.co/}}

\begin{figure}[t]
    \centering
    \includegraphics[width=\linewidth]{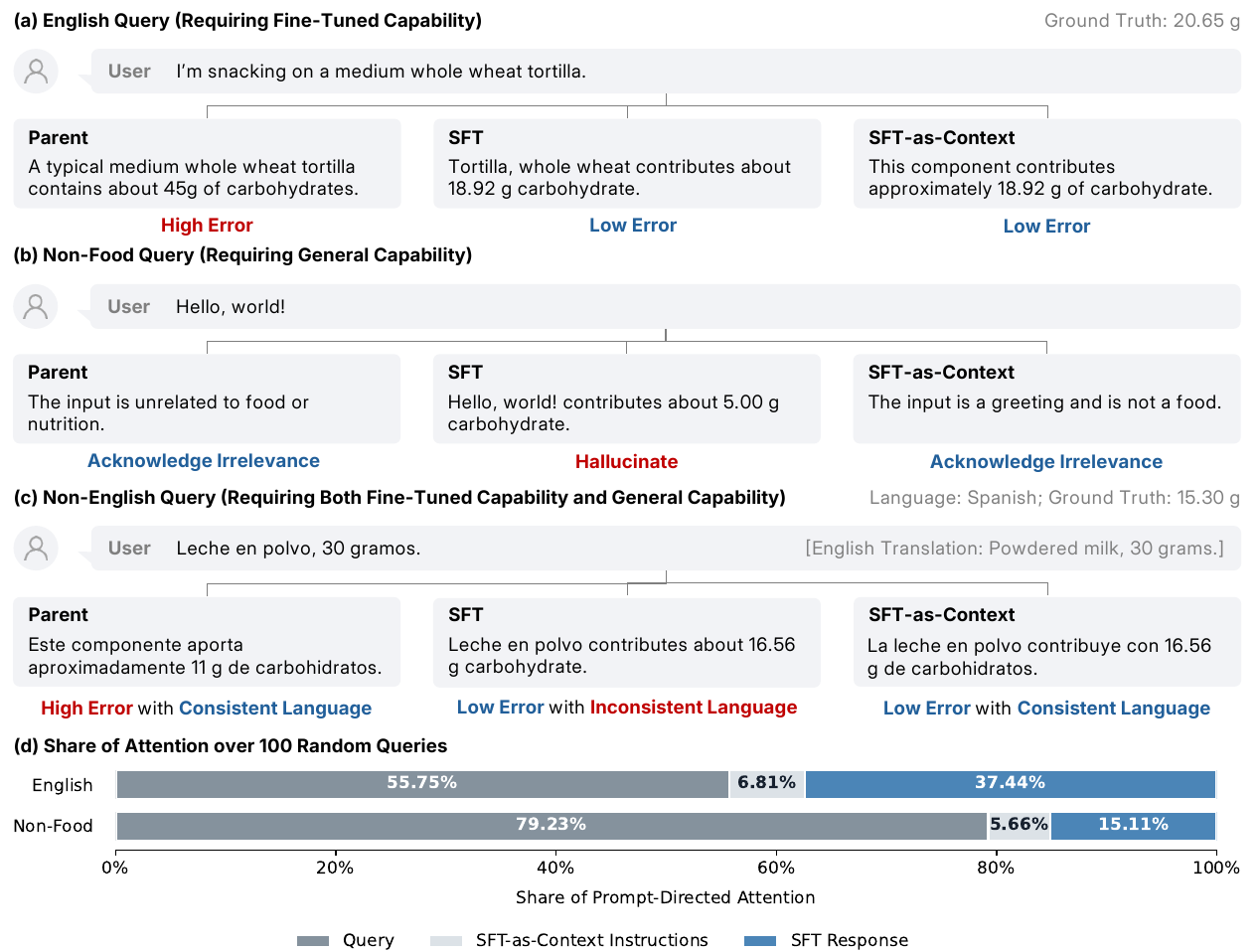}
    \caption{\textbf{SFT-as-context recovers both fine-tuned and general capabilities.} The task is to predict the nutritional values (carbohydrates, fat, protein, and energy) from a meal description, in the same language as the given query. Only the carbohydrate value is shown here as an example. Subfigure (a) shows that an SFT model trained on English data excels in fine-tuned capability. However, subfigures (b) and (c) show that it fails at acknowledging non-food queries and responding in the same language as the query. When a query requires both fine-tuned nutrition estimation capability and general multilingual capability, only SFT-as-context satisfies both requirements. Subfigure (d) shows that the parent model attends more to useful English SFT responses than to hallucinated SFT responses to non-food queries, suggesting that selective attention helps the parent model use the SFT response through in-context learning. We discuss details of the attention analysis in \Cref{sec:attention-analysis}.}
    \label{fig:capability_demonstration}
\end{figure}

Instead, we propose a \textit{training-free} method, SFT-as-context, which leverages the in-context learning capability of language models by providing a fine-tuned model's response as context to its parent model and prompting the parent model to generate the final answer.
Despite its simplicity, SFT-as-context preserves fine-tuning gains while mitigating forgetting across all \textbf{19 evaluated parent-fine-tuned model pairs and 11 benchmarks}, spanning mathematical reasoning, coding, nutrition estimation, and general capabilities.
To quantify performance recovery, we define the recovery rate for each capability as the fraction of the performance gap between the parent and SFT models that SFT-as-context recovers (Equation \ref{eqn:metric}). A recovery rate of 100\% means SFT-as-context matches the better model.
Using SFT checkpoints collected from Hugging Face, SFT-as-context achieves 95.0\% and 94.1\% recovery rates of the fine-tuning gains on AIME 2024 and LiveCodeBench~\citep{ICLR2025_94074dd5}, respectively. 
At the same time, it achieves an average recovery rate of 94.2\% for general capabilities across instruction-following, multilingual tasks, question answering, and planning.
Additionally, SFT-as-context combines fine-tuned and general capabilities within the same response, going beyond what either model can achieve alone. 
For example, on NutriBench-Non-English~\citep{dhaliwal2025nutribench}, as shown in \Cref{fig:capability_demonstration}(c), SFT-as-context  achieves recovery rates of 84.6\% for nutrition estimation and 95.2\% for language consistency.

More importantly, SFT-as-context requires only standard text input-output interfaces, making it applicable to \textit{closed-source} models without access to their parameters. 
This flexibility allows responses from a small open-source SFT model to serve as context for a stronger closed-source model, even when the latter is not the parent of the SFT model. 
Notably, this holds for a strong proprietary model: providing responses from a fine-tuned Qwen3-4B as context to Gemini 3.5 Flash reduces macro MAE from 17.44 to 14.08 on NutriBench-English, outperforming either model alone.

Going deeper, we study how SFT-as-context works from two perspectives. First, we use a Bayesian framework for in-context learning~\citep{xie2022an, hu2024unveiling} to provide theoretical guarantees that bound the differences in error between parent models, SFT models, and SFT-as-context. Second, we visualize the attention weights that the parent model assigns to the SFT output in the prompt. We find that the parent model selectively attends to the SFT output based on its effectiveness. For example, as shown in \Cref{fig:capability_demonstration}(d), the parent model allocates substantially more attention to the SFT response for nutrition queries than for unrelated queries (37.44\% versus 15.11\%). 
Taken together, these analyses suggest a single underlying principle of SFT-as-context: the parent model leverages in-context learning to selectively acquire fine-tuned capabilities from the SFT responses while preserving its own general capabilities.

\section{SFT Models Forget General Capabilities}

In this section, we describe our experimental setup and then present the results showing that SFT improves fine-tuned capabilities while degrading general capabilities.

\subsection{Experiment Setup}
We evaluate models fine-tuned for three domains: mathematical reasoning, coding, and nutrition estimation.
We refer to the capabilities that SFT improves as \textit{fine-tuned capabilities}, and the capabilities that the parent model already possessed as \textit{general capabilities}.
For example, in NutriBench, the fine-tuned capability is nutrition estimation, and general capabilities include multilingual understanding and instruction-following.

We evaluate 19 parent-SFT pairs, including 8 SFT models fine-tuned on English mathematical reasoning traces and 8 fine-tuned on coding tasks, all of which are popular models from Hugging Face (Appendix~\ref{sec:list-of-sft-models}). We also train 3 SFT models for nutrition estimation on NutriBench (Appendix~\ref{sec:nutrition-estimation-model-training}).

We first evaluate fine-tuned capabilities and general capabilities separately. 
For fine-tuned capabilities, we evaluate them on held-out data from the same distribution as the SFT training data: AIME 2024 for mathematics, LiveCodeBench~\citep{ICLR2025_94074dd5} for coding, and NutriBench-English, the English subset of NutriBench~\citep{dhaliwal2025nutribench}, for nutrition estimation. To evaluate general capabilities, we use IFEval~\citep{zhou2023instruction} and MGSM~\citep{shi2023language} for math models, and IFEval, SQuAD2.0~\citep{rajpurkar-etal-2018-know}, FQuAD2.0~\citep{heinrich-etal-2022-fquad2}, CoQA~\citep{reddy2019coqa}, and ACPBench~\citep{kokel2025acpbench} for coding models. 
For nutrition models, we use IFEval as a proxy for non-food queries, and we refer to this as NutriBench-Non-Food.

In addition, we evaluate tasks that require both fine-tuned and general capabilities in the same query:
\begin{itemize}[noitemsep, topsep=0pt, partopsep=0pt, leftmargin=*]
    \item \textbf{MathIF}~\citep{fu-etal-2026-scaling}: Queries combine mathematical problem solving with formatting requirements, such as writing the response in lowercase or using at least two highlighted sections.

    \item \textbf{LiveCodeBenchIF}: We construct this benchmark by appending a formatting instruction to each question from LiveCodeBench~\citep{ICLR2025_94074dd5}, which asks for exactly 3 comment lines.

    \item \textbf{NutriBench-Non-English}: The query is a non-English meal description (\Cref{fig:capability_demonstration}), and an English prompt template asks for an answer in the same language as the query.
\end{itemize}
We provide details of all benchmarks and evaluation metrics in Appendix~\ref{sec:experiment-details}.

\subsection{SFT Causes Forgetting}
\label{sec:sft-causes-forgetting}

\begin{table*}[t]
  \centering
  \caption{\textbf{SFT-as-context successfully mitigates the forgetting of general capabilities and maintains high fine-tuned capabilities compared to SFT across all benchmarks.} Higher
  is better except for NutriBench macro MAE, where lower is better. For IFEval, we report prompt-level strict accuracy (PSA).
  $^\dagger$ denotes joint-capability settings, where the same queries require both fine-tuned and general capabilities.
  The performance recovery rate ($R$) is defined in Equation~\ref{eqn:metric} and reported as a percentage. The metric results are averaged over 8 parent-SFT model pairs for mathematics, 8 for coding, and 3 for nutrition. Best results are bolded and second-best results are underlined.}
  \label{tab:combined-domain-results}
  \small
  \setlength{\tabcolsep}{6pt}
  \begin{tabular*}{\textwidth}{@{\extracolsep{\fill}}llcccc@{}}
    \toprule
    Benchmark & Metric & Parent & SFT & SFT-as-Context & $R$ ($\uparrow$) \\

    \midrule
    \rowcolor{black!10}[0pt][0pt]
    \multicolumn{6}{@{}c@{}}{\textbf{Mathematics}} \\

    \multicolumn{6}{@{}l}{\textbf{Fine-Tuned Capabilities}} \\

    AIME 2024
      & Accuracy
      & 11.3 & \textbf{56.0} & \underline{53.8} & 95.0 \\
    MathIF$^\dagger$
      & Math Correctness
      & 40.6 & \textbf{65.0} & \underline{64.0} & 95.9 \\

    \multicolumn{6}{@{}l}{\textbf{General Capabilities}} \\

    IFEval
      & PSA
      & \textbf{71.3} & 52.7 & \underline{69.8} & 91.6 \\

    MathIF$^\dagger$
      & Hard IF
      & \textbf{50.1} & 21.3 & \underline{45.8} & 85.1 \\

    MGSM
      & Answer-Prefix Accuracy
      & \textbf{90.6} & 50.2 & \underline{87.5} & 92.2 \\

    \midrule
    \rowcolor{black!10}[0pt][0pt]
    \multicolumn{6}{@{}c@{}}{\textbf{Coding}} \\

    \multicolumn{6}{@{}l}{\textbf{Fine-Tuned Capabilities}} \\

    LiveCodeBench
      & Pass@1
      & 20.1 & \textbf{55.5} & \underline{53.4} & 94.1 \\
    LiveCodeBenchIF$^\dagger$
      & Pass@1
      & 17.3 & \textbf{53.8} & \underline{51.2} & 92.9 \\

    \multicolumn{6}{@{}l}{\textbf{General Capabilities}} \\

    IFEval
      & PSA
      & \textbf{75.0} & 45.1 & \underline{74.8} & 99.4 \\

    LiveCodeBenchIF$^\dagger$
      & Exactly 3 Comment Lines
      & \textbf{31.5} & 19.1 & \underline{31.1} & 96.6 \\

    SQuAD2.0
      & Correctness
      & \textbf{86.1} & 62.5 & \underline{83.3} & 88.3 \\
    FQuAD2.0
      & Correctness
      & \textbf{81.0} & 66.2 & \underline{78.2} & 81.3 \\
    CoQA
      & Correctness
      & \textbf{95.9} & 62.2 & \underline{93.1} & 91.7 \\
    ACPBench
      & Overall Accuracy
      & \underline{77.0} & 65.9 & \textbf{78.8} & 116.2 \\

    \midrule
    \rowcolor{black!10}[0pt][0pt]
    \multicolumn{6}{@{}c@{}}{\textbf{Nutrition}} \\

    \multicolumn{6}{@{}l}{\textbf{Fine-Tuned Capabilities}} \\

    NutriBench-English
      & Macro MAE $\downarrow$
      & 35.3 & \textbf{16.4} & \underline{18.4} & 89.2 \\
    NutriBench-Non-English$^\dagger$
      & Macro MAE $\downarrow$
      & 85.6 & \textbf{59.6} & \underline{63.6} & 84.6 \\

    \multicolumn{6}{@{}l}{\textbf{General Capabilities}} \\

    NutriBench-Non-English$^\dagger$
      & Language Consistency
      & \textbf{77.3} & 1.9 & \underline{73.7} & 95.2 \\
    NutriBench-Non-Food
      & Acknowledge Non-Food
      & \textbf{98.8} & 22.6 & \underline{98.0} & 98.9 \\
      
    \bottomrule
  \end{tabular*}
\end{table*}

Consistent with prior observations of forgetting after SFT~\citep{chen2026retaining,huan2026does}, we find that the improvement of fine-tuned capabilities comes at the expense of general capabilities.
As shown in \Cref{tab:combined-domain-results}, SFT causes substantial forgetting of general capabilities, reducing general capability scores by 33.2 percentage points on average.
For example, among the three nutrition SFT models, Qwen3-8B shows the largest drop in acknowledging non-food queries, from 99.6\% to 10.7\%, catastrophically forgetting the instruction-following capability of its parent model.

This forgetting phenomenon becomes more problematic when a single query requires both fine-tuned and general capabilities, as the loss of either can cause the entire response to fail.
For example, each query in MathIF requires both the fine-tuned capability of solving the math problem correctly and the general capability of following additional response constraints, such as using only lowercase letters.
The SFT model significantly improves math correctness, from 40.6\% to 65.0\%, but fails to follow the format constraints. As a result, the hard instruction-following (Hard IF) score falls from 50.1\% to 21.3\%.
In one example of a polynomial-remainder problem (Appendix~\ref{sec:more-examples-of-forgetting-and-mitigation}), the SFT model answers correctly but uses the uppercase symbol $R$, violating the lowercase requirement. 
Similar problems occur on LiveCodeBenchIF and NutriBench-Non-English.

Taken together, fine-tuning substantially improves performance on the task the model is fine-tuned on, but at the expense of forgetting general capabilities that the parent model already possesses, e.g., instruction-following and multilingual understanding.

\begin{figure}[t]
    \centering
    \includegraphics[width=\linewidth]{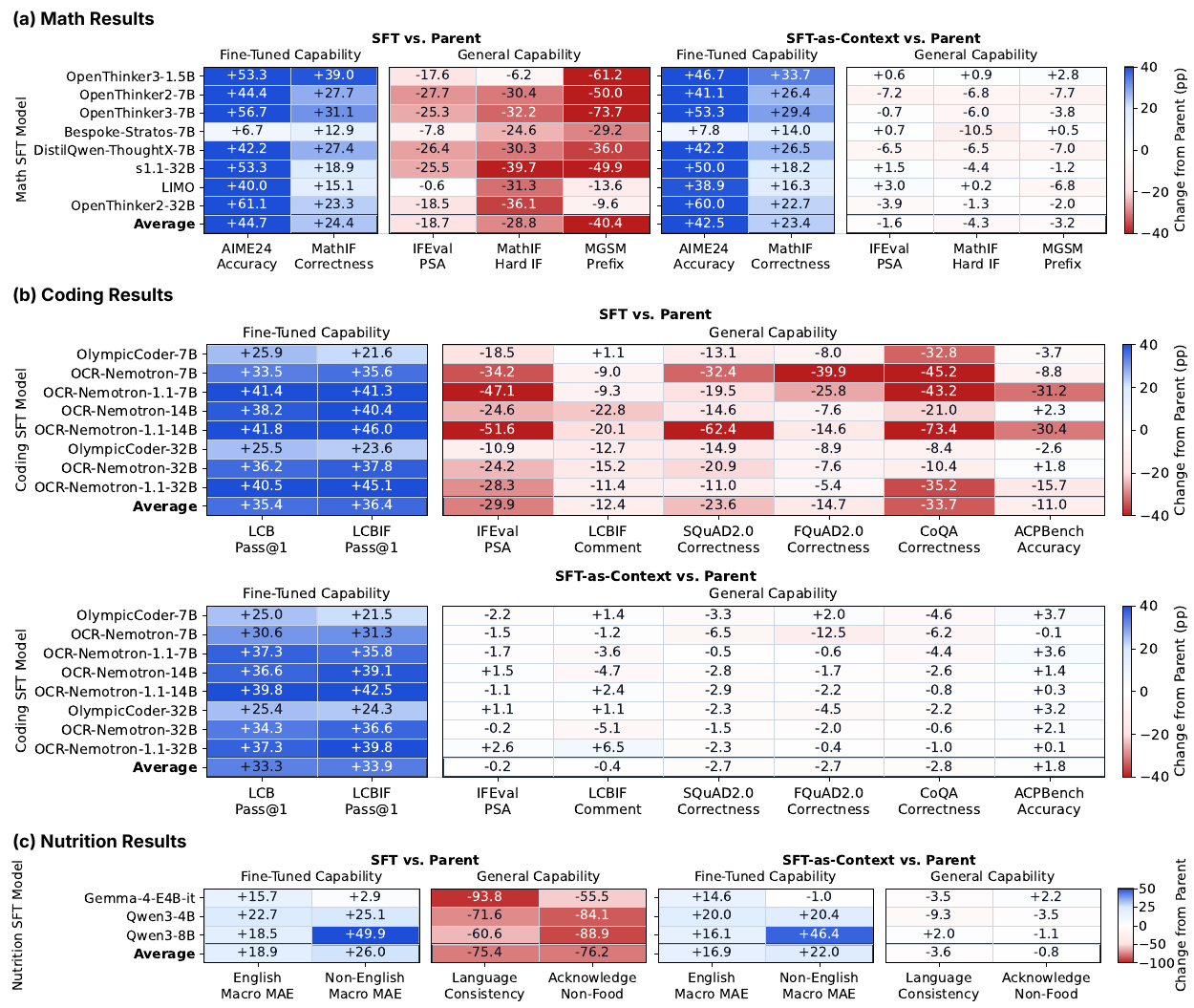}
    \caption{\textbf{Across individual model pairs, SFT-as-context maintains fine-tuned capabilities and mitigates forgetting of general capabilities.} In each heatmap, the first group of columns shows fine-tuned capabilities, and the second group of columns shows general capabilities. Each domain has two heatmaps: one showing the performance difference between the SFT model and the parent model (SFT vs. Parent), and the other showing the difference between SFT-as-context and the parent model (SFT-as-context vs. Parent). Darker blue indicates a larger improvement over the parent model, whereas darker red indicates more severe forgetting relative to the parent model. Color ranges are clamped to prevent large outliers from dominating the colors. For NutriBench macro MAE, the values are reversed to consistently show greater improvement in darker blue.}
    \label{fig:breakdown-heatmaps}
\end{figure}

\section{SFT-as-Context Successfully Mitigates Forgetting}
\label{sec:sft-as-context-successfully-mitigates-forgetting}

\paragraph{Method.}
\label{sec:sft-as-context-method}

To mitigate forgetting in SFT models, we propose SFT-as-context, which leverages the in-context learning ability of the parent model by providing the SFT model's response as context when generating the final response. Specifically, during inference, SFT-as-context answers each query through two sequential passes. In the first pass, the SFT model generates a response to the query. In the second pass, we construct the parent model's prompt from three parts: the original query, the SFT response from the first pass, and the SFT-as-context instructions. These instructions ask the parent model to treat the SFT response as a helpful answer unless there are concrete errors.

Formally, let \(x\) denote the input query, \(f_\sft\) the SFT model, and \(f_\parent\) its parent model. The SFT model first generates:
$
y_\sft = f_\sft(x).
$
We then construct the input to the parent model as \([x, y_\sft, I]\), where \(I\) denotes the SFT-as-context instructions. The SFT-as-context instructions are fixed for each domain (math, coding, or nutrition), agnostic to models and different fine-tuned or general capability benchmarks (Appendix~\ref{sec:sft-as-context-prompts}). The parent model generates the final response as: 
\begin{equation}
    y_\sftc = f_\parent(x, y_\sft, I).
\end{equation}
Thus, SFT-as-context provides the SFT model's response as context to the parent model, which then generates the final answer. Crucially, SFT-as-context is a novel, training-free approach to mitigating forgetting that requires neither access to model weights nor the fine-tuning process. It operates entirely through text input-output interfaces, making it applicable even to closed-source models.

\paragraph{Evaluation Metric.} We define the \textit{performance recovery rate} to quantify how much of the performance gap between the parent and SFT models is recovered by SFT-as-context for a given capability. Let $s_\parent$, $s_\sft$, and $s_\sftc$ denote the parent, SFT, and SFT-as-context scores (e.g., accuracy or MAE), respectively. For metrics where higher is better, the recovery rate $R$ is defined as
\begin{equation}\label{eqn:metric}
R = [s_\sftc-\min(s_\parent,s_\sft)] / |s_\parent-s_\sft| \times 100\%.
\end{equation}
A recovery rate of $0\%$ indicates that SFT-as-context matches the worse-performing model, while $100\%$ indicates that it matches the better-performing model. The recovery rate can exceed $100\%$ when SFT-as-context outperforms both models. For metrics where lower is better, such as MAE, the recovery rate is alternatively defined as $100\%-R$.

\paragraph{SFT-as-Context Combines the Strengths of Parent and SFT Models.} Despite its simplicity, SFT-as-context achieves performance close to the parent model on general capabilities while remaining close to the SFT model on fine-tuned capabilities, effectively combining the strengths of SFT and parent models, as shown in \Cref{tab:combined-domain-results}.  Across 19 parent-fine-tuned model pairs and 11 benchmarks, SFT-as-context preserves 92.0\% of the fine-tuning gains on average, while recovering 94.2\% of the general-capability performance gap between parent and SFT models.
More importantly, SFT-as-context can combine fine-tuned and general capabilities within the same response, which neither the parent nor the SFT model can achieve alone. We evaluate this setting on MathIF, LiveCodeBenchIF, and NutriBench-Non-English, where each query requires both fine-tuned and general capabilities, and we report the performance on both for each dataset in \Cref{tab:combined-domain-results}.
For the same MathIF example discussed in \Cref{sec:sft-causes-forgetting}, the SFT model provides a correct answer but uses uppercase letters, whereas the SFT-as-context response preserves the correct answer while satisfying the lowercase formatting constraint (Appendix~\ref{sec:mathif-and-livecodebenchif-general-failures}).

\Cref{fig:breakdown-heatmaps} shows that these trends extend across individual model pairs, including those with severe forgetting.
For OCR-Nemotron-1.1-14B, SFT reduces CoQA performance by 73.4 percentage points relative to the parent model, indicating severe forgetting in conversational question answering.
SFT-as-context brings this performance back to near the parent model's level, with a gap of only 0.8 percentage points, while preserving most of the gains on LiveCodeBench (39.8 out of 41.8 percentage points).
This highlights that SFT-as-context can mitigate severe forgetting while preserving substantial fine-tuning gains.

To understand the remaining performance gap, we zoom into MathIF, which shows the largest gap on the Hard IF metric (\Cref{fig:breakdown-heatmaps}). The dominant failure mode is that the small parent model, Qwen2.5-7B-Instruct (the shared parent of all 7B math SFT models), often returns only a boxed answer and thus violates formatting requirements of using specific words or sections (Appendix~\ref{sec:sft-as-context-failure-mode}). This is effectively mitigated by a larger parent: the 32B parent exhibits this failure far less often, yielding a smaller remaining gap for the 32B models than for the 7B ones (an average gap of $-1.8$ versus $-7.5$).

\section{Understanding the Mechanisms of SFT-as-Context}
\label{sec:mechansims}

To understand why SFT-as-context can mitigate forgetting while preserving the fine-tuned capabilities, we first analyze how LLMs combine capabilities from a theoretical perspective (\Cref{sec:theoretical-explanation}). Then, by analyzing attention patterns, we empirically show that in SFT-as-context, a parent LLM attends more to useful SFT responses and less to hallucinated context, suggesting a possible mechanism for recovering general capabilities while preserving fine-tuned capabilities (\Cref{sec:attention-analysis}).

\subsection{Theoretical Explanation}
\label{sec:theoretical-explanation}

We explain the effectiveness of SFT-as-context using a Bayesian framework for in-context learning \citep{xie2022an, hu2024unveiling}, building on earlier Bayesian models \citep{baum1966statistical, blei2003latent}.
Following \citet{xie2022an}, we model response generation as conditioned on a domain variable $\theta$, with output distribution $\Pr(Y\mid x,\theta)$.
Each such variable $\theta$ corresponds to a domain such as nutrition estimation, math, instruction-following, and multilinguality.

The parent model is trained using a large dataset across many different domains and therefore has many capabilities. 
Let $\Theta_{\parent}$ denote the parent model's domain set.
We assume that the SFT model is fine-tuned on a single domain $\theta_{\sft}$ and denote its query set by $D_{\sft}$.
Our analysis explains how SFT-as-context combines SFT expertise with the parent model's general capabilities.

For a query $x$, let $\Pr_{\parent}(\cdot\mid x)$ and $\Pr_{\sft}(\cdot\mid x)$ denote the parent and SFT response distributions, respectively, and define
$\Pr_{\sftc}(\cdot\mid x):=\Pr_{\parent}(\cdot\mid x,y_{\sft},I)$,
where $y_{\sft}$ is the SFT response and $I$ denotes the SFT-as-context instructions.
All variables are discrete.
For each domain, assume an optimal domain variable $\theta^*$ inducing the optimal response law
$\Pr_*(\cdot\mid x):=\Pr(\cdot\mid x,\theta^*)$.
For $m\in\{\parent,\sft,\sftc\}$, define the corresponding errors $\err_m
    :=\operatorname{KL}\!\left(
        \Pr_*(\cdot\mid x)\,\|\,\Pr_m(\cdot\mid x)
    \right).$
We model the parent and SFT-as-context distributions as mixtures:
\begin{align*}
    \Pr_{\parent}(y\mid x)
    &=
    \sum_{\theta\in\Theta_{\parent}}
    \Pr_{\parent}(y\mid x,\theta)
    \Pr_{\parent}(\theta\mid x), \\
    \Pr_{\sftc}(y\mid x)
    &=
    \sum_{\theta\in\Theta_{\parent}}
    \Pr_{\parent}(y\mid x,\theta)
    \Pr_{\sftc}(\theta\mid x,y_{\sft},I).
\end{align*}
Thus, conditioning on the SFT response and instructions changes
the domain posterior while preserving the per-domain conditional
output laws.
We refer to $\Pr_{\cdot}(\theta | x)$ as the posterior probability, which under the Bayesian framework is the probability of the domain variable conditioned on the query.
Intuitively, we model the LLM's generation process as predicting a domain based on the query with some probability and then generating the response conditioned on the domain variable \citep{xie2022an}.
We make the following assumptions for all $x$ where all logarithms are natural, and the total variation is defined as $\operatorname{TV}(P,Q) := \frac{1}{2}\sum_y |P(y)-Q(y)|$.
We present assumptions intuitively followed by the formalism.
\begin{enumerate}[leftmargin=*]
    \item The posterior probability of the SFT domain variable is low for the parent model. Formally, $ 0 < a_x := \Pr_{\parent}(\theta_{\sft} \mid x) \leq \epsilon$ for the input $x$ that belongs to the SFT domain, where $\theta_{\sft}$ is the optimal domain variable for the SFT domain.
    
    \item The posterior probability of the SFT domain variable given the SFT response is $\epsilon'$-close to 1. Formally, under the parent model, the SFT response \(y_{\sft}\) satisfies $\Pr_{\sftc}\!\left(\theta_{\sft}\mid x,\prompt,y_{\sft}\right)\ge 1-\epsilon'$, meaning that the posterior probability assigned to the optimal domain variable \(\theta_{\sft}\) is within \(\epsilon'\) of 1.
    
    \item The optimal response distribution is separated from that of the parent model by some margin \citep{xie2022an}. Let $Q_{\parent}$ denote the parent model's response distribution
conditioned on a non-SFT domain:
$
Q_{\parent}(y\mid x)
:= \Pr_{\parent}(y\mid x,\theta\neq\theta_{\sft}).
$ There exists $\epsilon'' > 0$ such that, for every $x \in D_{\sft}$: $
    \operatorname{TV}\!\left(
        \Pr_*(\cdot \mid x),Q_{\parent}(\cdot \mid x)
    \right)
    \geq \epsilon''$.

    \item The SFT model is an expert on the one domain it was fine-tuned on. Formally, the SFT domain variable $\theta_{\sft}$ is such that
$
    \Pr_{\sft}(y \mid \theta_{\sft},x)
    = \Pr_*(y \mid x)
$.
Furthermore, there exists $0 \leq \eta < 1$ such that,
for every $x \in D_{\sft}$, $c_x
    := \Pr_{\sft}(\theta_{\sft} \mid x)
    \geq 1-\eta.$

    \item The SFT-as-context domain posterior retains a minimum fraction
of the corresponding parent posterior. Formally, there exists
$0 \leq \rho_o < 1$ such that, for every out-of-domain query
and realized context under consideration,
\[
    \Pr\nolimits_{\sftc}(\theta \mid x,\prompt,y_{\sft})
    \geq
    (1-\rho_o)\Pr\nolimits_{\parent}(\theta \mid x),
    \qquad \forall \theta \in \Theta_{\parent}.
\]
\vspace{-0.5em}
\end{enumerate}

Under these assumptions, we have the following results for $x$ that is in domain for the SFT dataset:
\begin{theorem}[Error guarantees for in-domain queries]
\label{thm:sftc_guarantees}
Suppose that the shared conditional-output model and
in-domain realizability condition hold (Appendix~\ref{sec:sft_in_domain_proof}).
Let Assumptions~1--4 hold with
$0<\epsilon<1$, $0\leq\epsilon'<1$,
$0<\epsilon''\leq 1$, and $0\leq\eta<1$.
For every in-domain query $x\in D_{\sft}$ and
the realized context $(y_{\sft},I)$ satisfying these assumptions,
the following are true:
    \begin{equation}
    \label{eq:parent_improvement}
    \textbf{Improvement over the parent: }
    \err_{\parent}-\err_{\sftc}
        \geq
        2(1-\epsilon)^2(\epsilon'')^2
        + \log(1-\epsilon').
    \end{equation}
    \begin{equation}
    \textbf{Closeness to SFT: }
    \label{eq:sft_closeness}
        \log(1-\epsilon')
        \leq
        \err_{\sft}-\err_{\sftc}
        \leq
        -\log(1-\eta).
    \end{equation}
    \vspace{-2em}
\end{theorem}
The theorem identifies sufficient conditions under which
SFT-as-context improves in-domain performance compared to the parent model while
limiting degradation compared to the SFT model.
Intuitively, Equation \ref{eq:parent_improvement} shows that the error of the parent for in-domain queries is greater than the error of SFT-as-context under the Bayesian model of text generation.
Equation \ref{eq:sft_closeness} means that the errors of SFT and SFT-as-context are close to one another.
The proof can be found in Appendix~\ref{sec:theoretical-results-and-proofs}.
While these results support how SFT-as-context maintains fine-tuned capabilities (\Cref{fig:breakdown-heatmaps}), we also show how under this Bayesian framework it maintains general capabilities with the following theorem.

\begin{theorem}[Error guarantees for out-of-domain queries]
\label{thm:out-of-domain}
Suppose that the shared conditional-output model and
out-of-domain realizability condition hold (Appendix~\ref{sec:sft_out_of_domain_proof}).
For any \textbf{out-of-domain query} $x\in D_o$, where
$\theta_o\in\Theta_{\parent}$ and
$\theta_o\neq\theta_{\sft}$, suppose that
Assumption 5 holds with $0\leq\rho_o<1$ and that
$\err_{\parent}<\infty$, with errors measured
relative to the optimal output law on $D_o$.
Then, for every realized context satisfying this assumption: 
$
    \err_{\sftc}-\err_{\parent}
    \leq -\log(1-\rho_o).
$
\vspace{-0.5em}
\end{theorem}

The inequality above shows that for out-of-domain queries the SFT-as-context method performs similarly to the parent model by upper bounding the difference in error.
We have thus shown the conditions under which the errors of SFT-as-context are bounded relative to the SFT model on in-domain queries and relative to the parent model on out-of-domain queries.
Our theoretical results bolster our empirical claims: SFT-as-context remains close to the SFT model on fine-tuned capabilities while staying close to the parent model on general capabilities (\Cref{tab:combined-domain-results} and \Cref{fig:breakdown-heatmaps}).

\subsection{Attention Analysis}
\label{sec:attention-analysis}

SFT-as-context places the SFT model's response in the parent model's context. However, these responses can provide useful task-specific information or contain misleading hallucinations, as illustrated in \Cref{fig:capability_demonstration}. How does the parent model use these responses in SFT-as-context?

We investigate how the parent model uses SFT responses by analyzing attention weights, which characterize how strongly it attends to different prompt tokens during generation~\citep{clark-etal-2019-bert}.
We compare two settings: nutrition queries from NutriBench-English, for which SFT responses provide useful context, and non-food queries from NutriBench-Non-Food, for which SFT responses may contain hallucinated nutrition estimates. The model pair is Gemma-4-E4B-it and its SFT version. We randomly sample 100 queries from each category and measure the attention allocated to the SFT responses, averaging over response tokens, layers, and attention heads (Appendix~\ref{sec:attention-visualization}).

\Cref{fig:capability_demonstration}(d) shows that the parent model attends more to useful SFT responses than hallucinated ones, with a 22.33 percentage point difference in the average share of attention allocated to the SFT responses.
This selective attention suggests that the parent model can use useful information from SFT responses through in-context learning while limiting the influence of irrelevant context, providing a possible explanation for how SFT-as-context mitigates forgetting while preserving fine-tuned capabilities.

\section{Discussion}

Here, we extend SFT-as-context to a strong proprietary model, showing that responses from open-source SFT models can further improve its performance with only text input-output access (\Cref{sec:cooperation-with-proprietary-model-for-even-stronger-capability}). We also provide ablation studies to show that the recovery of both fine-tuned and general capabilities does not arise solely from the SFT-as-context instructions (\Cref{sec:ablation-studies}).

\subsection{SFT-as-Context Further Improves Proprietary Models}
\label{sec:cooperation-with-proprietary-model-for-even-stronger-capability}

Beyond combining the strengths of the parent and SFT models, SFT-as-context can further improve fine-tuned capabilities when the final response is generated by a stronger model rather than the original parent model.
To be more specific, we fine-tune an open-source model such as Qwen3 and use its response as context for Gemini, which further improves Gemini's performance.
This flexibility is unique to SFT-as-context because it requires only standard text input-output access, without access to model weights or the fine-tuning process.

As shown in \Cref{tab:nutribench-macro-mae-results}, using responses from open-source SFT models as context consistently improves the performance of a strong proprietary model on NutriBench-English, achieving lower MAE than either model alone. This result also shows that the proprietary model does not merely copy the SFT response during inference. Instead, it can combine task-specific information from the SFT response with its own capabilities to produce a more accurate final answer.

\subsection{Ablation Studies}
\label{sec:ablation-studies}

To test whether the gains of SFT-as-context come simply from the SFT-as-context instructions ($I$) and an additional inference pass, we evaluate two control settings on NutriBench: the parent model using its own response as context and the SFT model using its own response as context.

As shown in \Cref{tab:nutribench-ablation}, the parent model with its own context preserves strong general capabilities, but its English and non-English macro MAE remain at 36.6 and 77.4, compared with 18.4 and 63.6 for SFT-as-context.
Conversely, the SFT model with its own context retains most of its nutrition estimation gains but fails to recover general capabilities. Language consistency and acknowledgment of non-food queries are only 1.3\% and 17.9\%, compared with 73.7\% and 98.0\% for SFT-as-context.
These results suggest that the gains of SFT-as-context cannot be explained solely by the SFT-as-context instructions or an additional inference pass.

\begin{table*}[t]
  \centering
  \caption{\textbf{A strong proprietary model can use context from open-source SFT models to further improve fine-tuned capabilities.} On NutriBench-English, using the context from Qwen3-4B achieves the best macro MAE. Best results are bolded and second-best results are underlined.}
  \label{tab:nutribench-macro-mae-results}
  
  \scriptsize
  \begin{tabular*}{\textwidth}{@{\extracolsep{\fill}}lccccccc@{}}
    \toprule
    & \multicolumn{1}{c}{Proprietary Model}
    & \multicolumn{3}{c}{Open-Source SFT Models}
    & \multicolumn{3}{c}{SFT-as-Context} \\
    \cmidrule(lr){2-2}\cmidrule(lr){3-5}\cmidrule(lr){6-8}
    Metric
      & Gemini 3.5 Flash
      & Gemma-4-E4B-it
      & Qwen3-4B
      & Qwen3-8B
      & Gemma-4-E4B-it
      & Qwen3-4B
      & Qwen3-8B \\
    \midrule
    Macro MAE $\downarrow$
      & 17.44
      & 16.64
      & 16.46
      & 16.05
      & \underline{14.13}
      & \textbf{14.08}
      & 14.41 \\
    \bottomrule
  \end{tabular*}
\end{table*}

\begin{table*}[t]
  \centering
  \caption{\textbf{SFT-as-context achieves the best balance between fine-tuned and general capabilities.} Parent with parent context fails to achieve low English and non-English macro MAE, and SFT with SFT context still fails on language consistency and acknowledging non-food queries. Best results are bolded and second-best results are underlined.}
  \label{tab:nutribench-ablation}
  \scriptsize
  \setlength{\tabcolsep}{4pt}
  \begin{tabular*}{\textwidth}{@{\extracolsep{\fill}}lcccc@{}}
    \toprule
    & \multicolumn{2}{c}{Fine-Tuned Capabilities}
    & \multicolumn{2}{c}{General Capabilities} \\
    \cmidrule(lr){2-3}\cmidrule(lr){4-5}
    Inference Setting
      & \shortstack{English MAE $\downarrow$}
      & \shortstack{Non-English MAE $\downarrow$}
      & \shortstack{Language Consistency (\%) $\uparrow$}
      & \shortstack{Acknowledge Non-Food (\%) $\uparrow$} \\
    \midrule
    Parent
      & 35.3 & 85.6 & \underline{77.3} & \underline{98.8} \\
    Parent with Parent Context
      & 36.6 & 77.4 & \textbf{83.4} & \textbf{99.7} \\
    SFT
      & \textbf{16.4} & \textbf{59.6} & 1.9 & 22.6 \\
    SFT with SFT Context
      & \underline{17.3} & 64.4 & 1.3 & 17.9 \\
    \midrule
    SFT-as-Context
      & 18.4 & \underline{63.6} & 73.7 & 98.0 \\
    \bottomrule
  \end{tabular*}
\end{table*}

\section{Related Work}

\paragraph{\textbf{Forgetting and Mitigation.}}
Prior work mitigates forgetting during fine-tuning through regularization toward pretrained behavior or parameters, as in Learning without Forgetting~\citep{li2017learning} and RecAdam~\citep{chen2020recall}, replay-based methods such as LAMOL~\citep{sun2020lamol}, constraints on parameter updates~\citep{lopez2017gradient,ICLR2026_4e447acb}, or the choice of adaptation algorithm~\citep{biderman2024lora,chen2026retaining}. Recent work finds that on-policy reinforcement learning preserves parent capabilities better than SFT~\citep{chen2026retaining}; we therefore focus on SFT-induced forgetting, where forgetting is typically more severe. In principle, the same inference-time approach could be applied to a model fine-tuned with reinforcement learning (RLFT) by providing its response as context to the parent model.

\paragraph{\textbf{Inference-Time Model Combination.}}
A line of work related to our setting explores combining pretrained and fine-tuned models at inference time. Emulated Fine-Tuning (EFT)~\citep{mitchell2024an} and Proxy-Tuning~\citep{liu2024tuning} combine model predictions through token-level probabilities or logits. They require access to output distributions during decoding and do not specifically target forgetting.

More recently, concurrent work CPR~\citep{ki2026cpr} mitigates forgetting by training a router to select between parent and fine-tuned models during decoding. In contrast, SFT-as-context requires only text input-output access, using the fine-tuned model's response as context for the parent model.

\section{Conclusion}

In this paper, we propose SFT-as-context, a training-free method that leverages the in-context learning ability of LLMs to mitigate the forgetting of general capabilities in SFT models while preserving their fine-tuned capabilities. Comprehensive experiments across domains, tasks, and model sizes demonstrate the effectiveness of SFT-as-context. 
We further analyze its underlying mechanism through a Bayesian view of in-context learning and attention analysis, showing that useful SFT responses can steer the parent model toward the fine-tuned domain while having less influence on unrelated queries, thereby preserving general capabilities. Since the same framework could naturally extend beyond SFT to other post-training methods, such as reinforcement learning, we leave this as a promising future direction to combine capabilities acquired through post-training with the broad capabilities of pretrained models.

\subsection*{AI use statement}

In this work, we used generative AI tools for generating synthetic datasets, implementing methods, assisting with translation, cleaning and reformatting datasets, supporting qualitative and thematic data analysis, and interpreting results. We have not used generative AI tools for helping develop theoretical models or conceptual frameworks, formulating mathematical claims, providing critical ingredients for proving mathematical claims, assisting in the writing of proofs, proposing or refining hypotheses, and designing or providing feedback on research methodology or experiments. Additionally, we used generative AI tools for creating or modifying scientific figures or images, creating or editing software code, summarizing or analyzing existing literature, editing a research paper to improve readability, and identifying relevant literature. 

We have reviewed all AI-assisted work. We manually verify synthetically generated meal descriptions for NutriBench. We compare with officially released metric values for public models to ensure LLM-generated code correctly implements dataset formatting, inference, and scoring. We use AI agents to speed up the identification of failure patterns in the forgetting of general capabilities, and we then manually conduct qualitative and thematic data analysis. We read through each of the suggested papers returned from AI-assisted literature review.

We take responsibility for the final content of this work, including text, claims or artifacts produced with the aid of generative AI.

\subsection*{Reproducibility statement}

Most of our results rely on publicly available model checkpoints and benchmarks, with the following three exceptions. First, we train our own models for nutrition estimation (Appendix~\ref{sec:nutrition-estimation-model-training}). Second, we add multilingual queries to the NutriBench dataset for testing fine-tuned and general capabilities (Appendix~\ref{sec:experiment-details-benchmarks}). Third, the proposed method requires manually written SFT-as-context instructions, which we include in Appendix~\ref{sec:sft-as-context-prompts}. Hence, we have provided sufficient details to allow full reproduction of the results. For theoretical results and proofs, we include more details in Appendix~\ref{sec:theoretical-results-and-proofs}.

\bibliography{iclr2027_conference}
\bibliographystyle{iclr2027_conference}

\newpage
\appendix

\section*{Appendix}
\setlength{\cftbeforesecskip}{0.5em}
\cftsetindents{section}{0em}{1.8em}
\cftsetindents{subsection}{1em}{2.5em}
\etoctoccontentsline{part}{Appendix}
{
  \etocsettocstyle{}{}
  \localtableofcontents
}

\section{Experiment Details}
\label{sec:experiment-details}

We provide more experiment details in this section.

\subsection{Benchmarks}
\label{sec:experiment-details-benchmarks}

In this work, we use the following benchmarks:

\textbf{American Invitational Mathematics Examination 2024} (AIME 2024) contains 30 math questions. Each question requires multi-step reasoning to solve. The answer to each question is an integer from 000 to 999.

\textbf{MathIF}~\citep{fu-etal-2026-scaling} is a benchmark for evaluating the instruction-following capability of reasoning models. It consists of 420 questions drawn from various math datasets, and each instruction contains Python-verifiable constraints from four categories (length, lexical, format, affix).

\textbf{IFEval}~\citep{zhou2023instruction} is a benchmark that evaluates the instruction-following capability of language models. The benchmark contains explicit, verifiable instructions, such as constraints on formatting, length, keywords, and response structure.

\textbf{MGSM}~\citep{shi2023language} is a multilingual math benchmark that contains translated versions of the 250 grade-school math problems from the GSM8K dataset~\citep{cobbe2021training} in 10 languages.

\textbf{LiveCodeBench}~\citep{ICLR2025_94074dd5} is a coding benchmark that uses continuously updated questions to evaluate coding models. In this work, we use the \texttt{v4\_v5} subset of LiveCodeBench, containing 268 questions. This choice is consistent with the version originally used by the OlympicCoder-7B and OlympicCoder-32B model series~\citep{penedo2025olympiccoder}.

\textbf{LiveCodeBenchIF} is a benchmark we constructed, in a similar fashion to MathIF, where we append a short formatting requirement to the end of the original questions from LiveCodeBench. To use a Python-verifiable constraint and to minimize the impact of the formatting instruction on the code structure, we only specify the number of comment lines to be 3 in the formatting instruction (Appendix~\ref{sec:task-prompts}). 

\textbf{SQuAD2.0}~\citep{rajpurkar-etal-2018-know} is a reading comprehension benchmark. The answer to every question is a segment of text from the corresponding reading passage, or the question might be unanswerable. We use a random sample of 1,187 (10\% of the total size) questions from the validation set. Of the 1,187 questions, 594 are answerable, and 593 are unanswerable.

\textbf{FQuAD2.0}~\citep{heinrich-etal-2022-fquad2} is a French-native reading comprehension benchmark, similar to SQuAD2.0. The benchmark contains 800 questions in its test set, with 400 answerable and 400 unanswerable questions.

\textbf{CoQA}~\citep{reddy2019coqa} is a benchmark designed to evaluate a model's ability to answer questions in the context of an ongoing conversation. We use the 500 questions in the public validation dataset. Following the implementation of the Language Model Evaluation Harness~\citep{eval-harness}, we evaluate the answer of the last dialogue turn.

\textbf{ACPBench}~\citep{kokel2025acpbench} is a benchmark designed to evaluate the reasoning capabilities of large language models across action, change, and planning. The benchmark consists of 7 reasoning tasks over 13 domains, requiring both single-step and multi-step reasoning capabilities. We evaluate on the officially released version where the answer is a boolean value.

We construct three evaluation subsets based on NutriBench~\citep{dhaliwal2025nutribench}.
\begin{itemize}[leftmargin=*]
    \item \textbf{NutriBench-English} is sampled directly from the original English NutriBench data. We randomly sample 2,000 English queries.
    
    \item \textbf{NutriBench-Non-English} is constructed using GPT-4o translations based on food items from the WHO source used in NutriBench. We sample 2,018 non-English queries to maximize language coverage, randomly sampling up to 200 queries per language and including all available queries for languages with fewer than 200 queries.
    
    \item \textbf{NutriBench-Non-Food} uses IFEval queries as a proxy for non-food inputs. To simulate deployment scenarios where users may provide non-food queries, we insert IFEval queries into the nutrition prediction template.
\end{itemize}

\subsection{List of SFT Models}
\label{sec:list-of-sft-models}

\Cref{tab:math-and-coding-model-pairs} shows the pairs of parent and SFT models we use for our experiments. For math and coding, the parent models are instruction-tuned versions from the Qwen2.5 family~\citep{yang2024qwen25} and Qwen2.5-Coder family~
\citep{qwen25coder}. More details for nutrition estimation model pairs are provided in Appendix~\ref{sec:nutrition-estimation-model-training}.

As a conceptual shorthand, we use ``pretrained model'' throughout our discussion to refer to the off-the-shelf instruction-tuned model that serves as the starting point for the task-specific supervised fine-tuning. We note that this differs from the stricter usage of \emph{pretrained} to denote a base model prior to instruction tuning.

\begin{table*}[t]
  \centering
  \caption{\textbf{Parent and SFT model pairs used in math, coding, and nutrition experiments.} All parent models are instruction-tuned. Models are listed in the order of increasing number of parameters. We use the abbreviation OCR for OpenCodeReasoning in the main paper, consistent with the choice of the original paper.}
  \label{tab:math-and-coding-model-pairs}
  \small
  \setlength{\tabcolsep}{8pt}
  \begin{tabular*}{\textwidth}{@{\extracolsep{\fill}}ll@{}}
    \toprule
    Parent Model & SFT Model \\
    \midrule

    \rowcolor{black!10}[0pt][0pt]
    \multicolumn{2}{@{}c@{}}{\textbf{Mathematics}} \\
    \addlinespace[2pt]
    Qwen2.5-1.5B-Instruct
      & OpenThinker3-1.5B~\citep{guha2026openthoughts}\\
    Qwen2.5-7B-Instruct
      & OpenThinker2-7B~\citep{guha2026openthoughts}\\
    Qwen2.5-7B-Instruct
      & OpenThinker3-7B~\citep{guha2026openthoughts}\\
    Qwen2.5-7B-Instruct
      & Bespoke-Stratos-7B~\citep{bespokelabs2025bespokestratos7b}\\
    Qwen2.5-7B-Instruct
      & DistilQwen-ThoughtX-7B~\citep{cai-etal-2026-reasoning}\\
    Qwen2.5-32B-Instruct
      & s1.1-32B~\citep{muennighoff-etal-2025-s1}\\
    Qwen2.5-32B-Instruct
      & LIMO~\citep{ye2025limo}\\
    Qwen2.5-32B-Instruct
      & OpenThinker2-32B~\citep{guha2026openthoughts}\\

    \midrule
    \rowcolor{black!10}[0pt][0pt]
    \multicolumn{2}{@{}c@{}}{\textbf{Coding}} \\
    \addlinespace[2pt]
    Qwen2.5-Coder-7B-Instruct
      & OlympicCoder-7B~\citep{penedo2025olympiccoder}\\
    Qwen2.5-7B-Instruct
      & OpenCodeReasoning-Nemotron-7B~\citep{ahmad2025opencodereasoning}\\
    Qwen2.5-7B-Instruct
      & OpenCodeReasoning-Nemotron-1.1-7B~\citep{ahmad2025opencodereasoning}\\
    Qwen2.5-14B-Instruct
      & OpenCodeReasoning-Nemotron-14B~\citep{ahmad2025opencodereasoning}\\
    Qwen2.5-14B-Instruct
      & OpenCodeReasoning-Nemotron-1.1-14B~\citep{ahmad2025opencodereasoning}\\
    Qwen2.5-Coder-32B-Instruct
      & OlympicCoder-32B~\citep{penedo2025olympiccoder}\\
    Qwen2.5-32B-Instruct
      & OpenCodeReasoning-Nemotron-32B~\citep{ahmad2025opencodereasoning}\\
    Qwen2.5-32B-Instruct
      & OpenCodeReasoning-Nemotron-1.1-32B~\citep{ahmad2025opencodereasoning}\\

    \midrule
    \rowcolor{black!10}[0pt][0pt]
    \multicolumn{2}{@{}c@{}}{\textbf{Nutrition}} \\
    \addlinespace[2pt]
    Gemma-4-E4B-it
      & Ours (Appendix~\ref{sec:nutrition-estimation-model-training})\\
    Qwen3-4B
      & Ours (Appendix~\ref{sec:nutrition-estimation-model-training})\\
    Qwen3-8B
      & Ours (Appendix~\ref{sec:nutrition-estimation-model-training})\\
    \bottomrule
  \end{tabular*}
\end{table*}

\subsection{Nutrition Estimation Model Training}
\label{sec:nutrition-estimation-model-training}

\paragraph{Models.}
We fine-tune three open-weight instruction-tuned models with full-parameter SFT (no LoRA or other adapters): \texttt{gemma-4-E4B-it}~\citep{gemma4}, \texttt{Qwen3-4B}, and \texttt{Qwen3-8B}~\citep{qwen3}.

\paragraph{Training data.}
All runs use the same data from NutriBench~\citep{dhaliwal2025nutribench}.
We sampled 40,000 English-only meal descriptions for training and 1,000 development examples.

\paragraph{Optimization.}
We train in bfloat16 with gradient checkpointing on 8$\times$NVIDIA A100 80\,GB GPUs using FSDP.
The per-device batch size is 2 with 16 gradient-accumulation steps (effective batch 256), and we use AdamW with a learning rate of $1\times10^{-5}$, no weight decay, a linear schedule, a warmup ratio of 0.03, and 10 epochs, i.e., 1{,}570 optimizer steps.

\subsection{Inference Settings}

For all benchmarks, we use a permissive decoding budget to leave the models sufficient room to generate complete responses. We set the maximum total sequence length, including both the prompt and response, to 32,768 tokens, substantially higher than the typical generation limits used for the selected benchmarks. We adopt this large budget because the math and coding SFT models are predominantly fine-tuned on long reasoning traces. For NutriBench only, we additionally cap the maximum number of response tokens at 8,192.

For each benchmark, we use the same decoding parameters across all models to ensure a fair comparison. We use greedy decoding for most benchmarks. For benchmarks where sampling is commonly used, we follow the officially recommended sampling parameters. For these non-greedy benchmarks, we report performance averaged over three random seeds. For SFT-as-context, each of the three inference runs uses a different context generated by the SFT model.

\subsection{Metrics}

For most of the benchmarks, we use the official evaluation setup. For IFEval, we report the prompt-level strict accuracy according to the official evaluation. For MathIF, we report Hard IF according to the official evaluation. 

For SQuAD2.0, FQuAD2.0, and CoQA, we use LLM-as-a-judge to avoid the unknown biases introduced by logit-based evaluation~\citep{hua-etal-2025-flaw}. We use the LLM-as-a-judge prompt recommended by~\citet{hua-etal-2025-flaw} and Gemini 3.5 Flash as the judge.

For MGSM, we check for the rate of using the answer prefix in the correct language. The prefix is the word for ``answer'' in each non-English language.

For NutriBench, macro MAE is defined as the mean absolute error (MAE) averaged across carbohydrates, fat, protein, and energy. For language consistency and acknowledgment of non-food queries, we heuristically check the presence of phrases such as ``contribute about'' and ``contributes about'' in the response, which indicates an English response and a hallucinated response for these two categories, respectively.

\section{Prompts}

We provide the task prompts and SFT-as-context prompts in this section.

\subsection{Task Prompts}
\label{sec:task-prompts}

Other than LiveCodeBenchIF and NutriBench, we use the official prompt template for evaluation of each benchmark. The custom prompt for LiveCodeBenchIF is shown below, where the only addition other than the original task prompt is a simple sentence asking for 3 lines of comments.

\begin{tcolorbox}[title=LiveCodeBenchIF Prompt Template, colback=gray!5, colframe=gray, fonttitle=\bfseries, breakable]
\{original\_task\_prompt\}\\

Use exactly three lines of comments in your code.
\end{tcolorbox}

The prompt template for NutriBench is shown below. The meal description is inserted at \{meal\_description\}. In real deployment, non-food queries may be inserted at \{meal\_description\} as well, and the model is expected to respond with a human\_reasoning field that acknowledges that there is no food in the query. The prompt template explicitly asks for language consistency and the acknowledgment of non-food queries. However, with this prompt for training and inference, SFT models fail to preserve these general capabilities. 

\begin{tcolorbox}[title=NutriBench Prompt Template, colback=gray!5, colframe=gray, fonttitle=\bfseries, breakable]
You are a nutrition estimation assistant.\\

Your task is to estimate the total nutritional content of the following food or meal description.\\

Food or meal description:\\
``\{meal\_description\}''\\

Before estimating, first determine what kind of input you received. If the description is written in a language other than English, write the entire \textasciigrave human\_reasoning\textasciigrave\ in that same language without switching to English. If the input is unrelated to food or nutrition, explicitly identify it as irrelevant, do not invent a meal, and set all four numeric fields to 0. If it mentions food but provides no quantity or serving size, explicitly state in \textasciigrave human\_reasoning\textasciigrave\ that the amount is unspecified and that the numerical result requires an assumption; then provide your best estimate.\\

Output format:\\
- You may write normal reasoning before the final prediction, including any internal matching, portion, scaling, and summing rationale.\\
- End your response with the final prediction JSON object wrapped exactly in these tags:\\
\textless nutrition\_estimate\_json\textgreater\\
\{\\
\hspace*{1em}``human\_reasoning'': ``brief explanation for a human, in the same language as the meal description'',\\
\hspace*{1em}``carbohydrate\_g'': number,\\
\hspace*{1em}``fat\_g'': number,\\
\hspace*{1em}``protein\_g'': number,\\
\hspace*{1em}``energy\_kcal'': number\\
\}\\
\textless /nutrition\_estimate\_json\textgreater\\
- Keep ``human\_reasoning'' compact, but include the nutrition contribution of each main component. Do not merely say that items were estimated and summed.\\
- For multi-component meals, use a style like: ``The meal consists of A, B, and C. A contributes about ...; B contributes about ...; C contributes about ....''\\
- The per-component nutrition in ``human\_reasoning'' should cover carbohydrate, fat, protein, and energy when possible, using the same units as the JSON fields.\\

Guidelines:\\
- Estimate the total amount for the entire described food or meal.\\
- Use grams for carbohydrate, fat, and protein.\\
- Use kilocalories for energy.
\end{tcolorbox}

\subsection{SFT-as-Context Prompts}
\label{sec:sft-as-context-prompts}

For SFT-as-context, we use two versions of prompts, one for NutriBench, another for math and coding models. In these prompts, the \{task\_prompt\} is the original query in the benchmarks. The nutrition estimation JSON replaces \{nutrition\_estimate\_json\}. The model response after stripping the thinking block replaces \{sft\_response\_thinking\_stripped\}.

For math and coding, the original task prompt is repeated at the end of the SFT-as-context instructions. This design choice is justified by an additional ablation study in Appendix~\ref{sec:additional-ablations-studies}.

\begin{tcolorbox}[title=NutriBench SFT-as-Context Prompt Template, colback=gray!5, colframe=gray, fonttitle=\bfseries, breakable]
\{task\_prompt\}
\\\\
The candidate below usually gives a better nutrition estimate than an independent guess. Treat it as the default answer. Use its numeric values unless there is an obvious hard error: it is for a different food, it ignores a major listed component, it uses a quantity that conflicts with the original description, or its energy is impossible for the stated amount. Do not adjust merely because your own rough estimate is different; only override a specific candidate value when you can point to the original text that makes it impossible.
\\\\
\{nutrition\_estimate\_json\}
\end{tcolorbox}

\begin{tcolorbox}[title=Math and Coding SFT-as-Context Prompt Template, colback=gray!5, colframe=gray, fonttitle=\bfseries, breakable]
Produce a response that fully satisfies the original task.\\
\\
\textless task\textgreater\\
\{task\_prompt\}\\
\textless /task\textgreater\\

\textless reference\textgreater\\
\{sft\_response\_thinking\_stripped\}\\
\textless /reference\textgreater\\

Treat the original task as the authoritative specification. Treat the reference as a candidate final answer, not as a draft that needs improvement.\\

If the reference already satisfies the original task, return it unchanged. Only change it when you identify a concrete correctness error or a violation of an explicit requirement in the original task. In that case, make the smallest necessary correction and preserve everything else.\\

Preserve the reference's scope, granularity, format, and working approach. Do not add optional material, rewrite for style, broaden the answer, or replace a working solution merely because another solution seems preferable.\\

The final response must be self-contained and must follow every requirement of the original task. Return only that response without mentioning the reference or this process.\\

Again, the original task is:\\

\textless task\textgreater\\
\{task\_prompt\}\\
\textless /task\textgreater\\

Use the reference unchanged unless a specific error or explicit requirement violation requires a minimal correction. Return only the final response.
\end{tcolorbox}

Additionally, to prevent degenerate or excessively long SFT responses from exhausting the parent model's available context window, we omit the SFT response if any of the following conditions holds:
\begin{enumerate}[leftmargin=*]
    \item The SFT model exhausts its allocated token budget without producing a final answer.
    \item The SFT response exhibits pathological repetition of token sequences.
    \item The combined length of the query, SFT response, and prompt template leaves no generation budget for the parent model.
\end{enumerate}

In these cases, we fall back to generating a response from the parent model using the original query alone. The fallback rate is generally low under the permissive token budget. For example, the fallback rate of MathIF is only 3.73\%, corresponding to 376 of 10,080 responses, across eight model pairs and seeds 0, 1, and 2. The gaps on fine-tuned and general capabilities are over 20 percentage points, much larger than 3.73\%. Hence, the improved capabilities of SFT-as-context cannot be attributed to the fallback conditions. 

\section{More Examples of Forgetting and Mitigation}
\label{sec:more-examples-of-forgetting-and-mitigation}

In this section, we first provide example queries and responses to demonstrate forgetting of SFT models and the mitigation by SFT-as-context (Appendix~\ref{sec:mathif-and-livecodebenchif-general-failures}). Following the discussion in \Cref{sec:sft-as-context-successfully-mitigates-forgetting}, we also provide examples from a case study of the largest remaining gap of general capability, which is from the MathIF benchmark (Appendix~\ref{sec:sft-as-context-failure-mode}).

\subsection{MathIF and LiveCodeBenchIF General Failures}
\label{sec:mathif-and-livecodebenchif-general-failures}

Here, we provide two examples, one from MathIF and one from LiveCodeBenchIF. Each example contains four parts: the query, the parent model response, the SFT model response, and the SFT-as-context response. In both examples, the parent models are weaker on fine-tuned capabilities, providing incorrect answers. The SFT models are weaker on general capabilities, ignoring the explicit formatting instructions. The SFT-as-context method generates a response with the correct answer and formatting, outperforming both parent and SFT models.

The MathIF example query, parent model response, SFT model response, and SFT-as-context response are shown below. The parent model is Qwen2.5-32B-Instruct, and the SFT model is OpenThinker2-32B. The parent uses the correct remainder-theorem strategy but incorrectly computes $f(1)=-5$ and $f(-1)=11$, leading to a wrong answer of $-8x+3$. Its response is lowercase, satisfying the formatting constraint. The SFT model correctly computes the answer but uses uppercase letters both in prose and notations (using the uppercase $R$ for remainder). SFT-as-context preserves the correct derivation and result, while using lowercase prose and notation $r$.

\begin{tcolorbox}[title=MathIF Example Query, colback=gray!5, colframe=gray, fonttitle=\bfseries, breakable]
Let $f(x) = x^{10}+5x^9-8x^8+7x^7-x^6-12x^5+4x^4-8x^3+12x^2-5x-5$.\\

Without using long division (which would be horribly nasty!), find the remainder when $f(x)$ is divided by $x^2-1$. Let's think step by step and output the final answer within \textbackslash boxed\{\}. Your entire response should be in English, and in all lowercase letters. No capital letters are allowed.
\end{tcolorbox}

\begin{tcolorbox}[title=Parent Response, colback=gray!5, colframe=gray, fonttitle=\bfseries, breakable]
we start by recognizing that when we divide a polynomial $f(x)$ by $x^2-1$, the remainder $r(x)$ will be a polynomial of degree less than 2. thus, we can express $r(x)$ as $ax+b$, where $a$ and $b$ are constants to be determined.

since $f(x) = (x^2-1)q(x) + r(x)$ for some quotient $q(x)$ and remainder $r(x) = ax + b$, we can substitute $x = 1$ and $x = -1$ into $f(x)$ to solve for $a$ and $b$. this is because $x^2-1$ becomes 0 when $x = 1$ or $x = -1$, simplifying the equation to just $r(x)$.

substituting $x = 1$:

\[ f(1) = (1)^{10} + 5(1)^9 - 8(1)^8 + 7(1)^7 - (1)^6 \]
\[- 12(1)^5 + 4(1)^4 - 8(1)^3 + 12(1)^2 - 5(1) - 5 = a(1) + b \]

\[ f(1) = 1 + 5 - 8 + 7 - 1 - 12 + 4 - 8 + 12 - 5 - 5 = a + b \]

\[ f(1) = -5 = a + b \]

substituting $x = -1$:

\[ f(-1) = (-1)^{10} + 5(-1)^9 - 8(-1)^8 + 7(-1)^7 - (-1)^6 - 12(-1)^5\]
\[ + 4(-1)^4 - 8(-1)^3 + 12(-1)^2 - 5(-1) - 5 = a(-1) + b \]

\[ f(-1) = 1 - 5 - 8 - 7 - 1 + 12 + 4 + 8 + 12 + 5 - 5 = -a + b \]

\[ f(-1) = 11 = -a + b \]

now we have a system of equations:
\[ a + b = -5 \]
\[ -a + b = 11 \]

adding these equations:
\[ (a + b) + (-a + b) = -5 + 11 \]
\[ 2b = 6 \]
\[ b = 3 \]

substituting $b = 3$ into one of the original equations:
\[ a + 3 = -5 \]
\[ a = -8 \]

thus, the remainder when $f(x)$ is divided by $x^2-1$ is $-8x+3$.

\[\boxed{-8x+3}\]
\end{tcolorbox}

\begin{tcolorbox}[title=SFT Response, colback=gray!5, colframe=gray, fonttitle=\bfseries, breakable]
To find the remainder when \( f(x) \) is divided by \( x^2 - 1 \), we use the Remainder Theorem for quadratic divisors. The remainder will be a linear polynomial \( R(x) = ax + b \). Evaluating \( f(x) \) at the roots of \( x^2 - 1 \) (which are \( x = 1 \) and \( x = -1 \)) gives us:

\[
\begin{cases}
f(1) = a(1) + b = a + b = -10 \\
f(-1) = a(-1) + b = -a + b = 16
\end{cases}
\]

Solving this system:

1. Subtract the second equation from the first:
   \[
   (a + b) - (-a + b) = -10 - 16 \implies 2a = -26 \implies a = -13
   \]
2. Substitute \( a = -13 \) into \( a + b = -10 \):
   \[
   -13 + b = -10 \implies b = 3
   \]

Thus, the remainder is \( R(x) = -13x + 3 \).

\[
\boxed{-13x + 3}
\]
\end{tcolorbox}

\begin{tcolorbox}[title=SFT-as-Context Response, colback=gray!5, colframe=gray, fonttitle=\bfseries, breakable]
to find the remainder when \( f(x) \) is divided by \( x^2 - 1 \), we use the remainder theorem for quadratic divisors. the remainder will be a linear polynomial \( r(x) = ax + b \). evaluating \( f(x) \) at the roots of \( x^2 - 1 \) (which are \( x = 1 \) and \( x = -1 \)) gives us:

\[
\begin{cases}
f(1) = a(1) + b = a + b = -10 \\
f(-1) = a(-1) + b = -a + b = 16
\end{cases}
\]

solving this system:

1. subtract the second equation from the first:
   \[
   (a + b) - (-a + b) = -10 - 16 \implies 2a = -26 \implies a = -13
   \]
2. substitute \( a = -13 \) into \( a + b = -10 \):
   \[
   -13 + b = -10 \implies b = 3
   \]

thus, the remainder is \( r(x) = -13x + 3 \).

\[
\boxed{-13x + 3}
\]
\end{tcolorbox}

The LiveCodeBenchIF example query, parent model response, SFT model response, and SFT-as-context response are shown below. The parent model is Qwen2.5-32B-Instruct, and the SFT model is OpenCodeReasoning-Nemotron-32B. In the example query, the last sentence specifying the number of comments is the only addition from us to the original problem description (Appendix~\ref{sec:experiment-details-benchmarks}). The parent response has three comment lines but sorts integers descending with \texttt{nums.sort(reverse=True)}. For [1,2,3], this gives 11\textbar{}10\textbar{}1 = 11101 (29), while the correct maximum is 11\textbar{}1\textbar{}10 = 11110 (30). The SFT response is correct but fully ignores the instruction of using 3 comments. The SFT-as-context response keeps the same executable Python structure from the SFT context and adds exactly three comment lines.

\begin{tcolorbox}[title=LiveCodeBenchIF Example Query, colback=gray!5, colframe=gray, fonttitle=\bfseries, breakable]
You will be given a question (problem specification) and will generate a correct Python program that matches the specification and passes all tests.\\

Question: You are given an array of integers nums of size 3.\\
Return the maximum possible number whose binary representation can be formed by concatenating the binary representation of all elements in nums in some order.\\
Note that the binary representation of any number does not contain leading zeros.\\
 
Example 1:\\

Input: nums = [1,2,3]\\
Output: 30\\
Explanation:\\
Concatenate the numbers in the order [3, 1, 2] to get the result ``11110'', which is the binary representation of 30.\\

Example 2:\\

Input: nums = [2,8,16]\\
Output: 1296\\
Explanation:\\
Concatenate the numbers in the order [2, 8, 16] to get the result ``10100010000'', which is the binary representation of 1296.\\

Constraints:\\

nums.length == 3\\
1 \textless = nums[i] \textless = 127\\

You will use the following starter code to write the solution to the problem and enclose your code within delimiters.\\
\textasciigrave\textasciigrave\textasciigrave python\\
class Solution:\\
\hspace*{2em}def maxGoodNumber(self, nums: List[int]) -\textgreater{} int:\\
\textasciigrave\textasciigrave\textasciigrave\\

Use exactly three lines of comments in your code.
\end{tcolorbox}

\begin{tcolorbox}[title=Parent Response, colback=gray!5, colframe=gray, fonttitle=\bfseries, breakable]
\includegraphics[width=\linewidth]{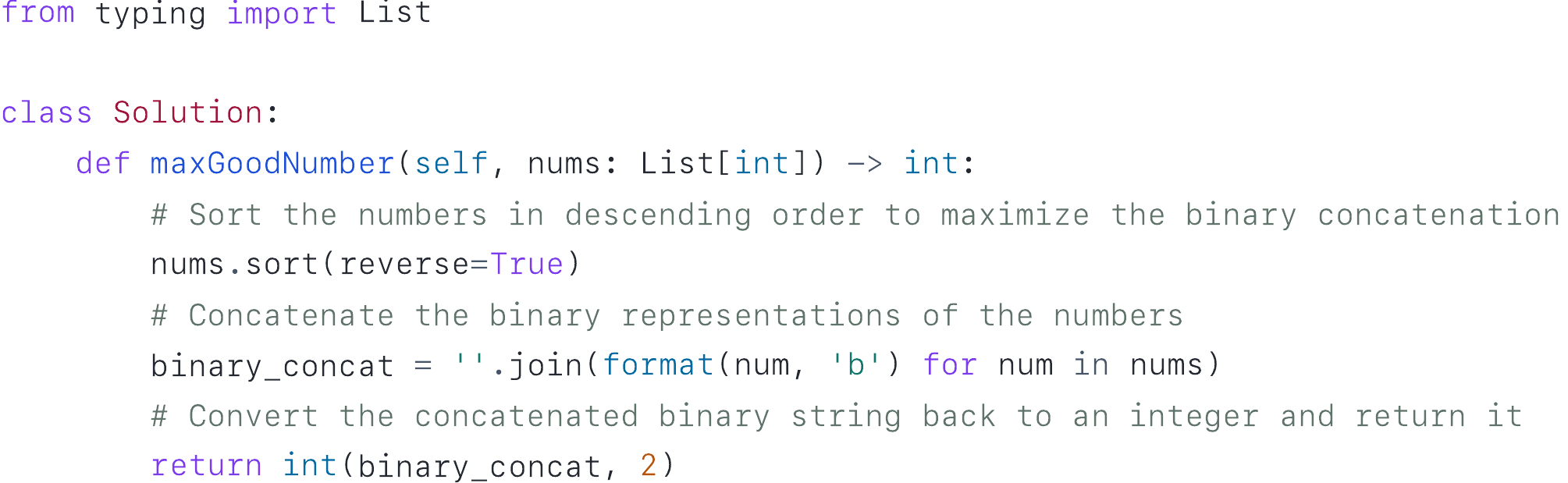}
\end{tcolorbox}

\begin{tcolorbox}[title=SFT Response, colback=gray!5, colframe=gray, fonttitle=\bfseries, breakable]
\includegraphics[width=\linewidth]{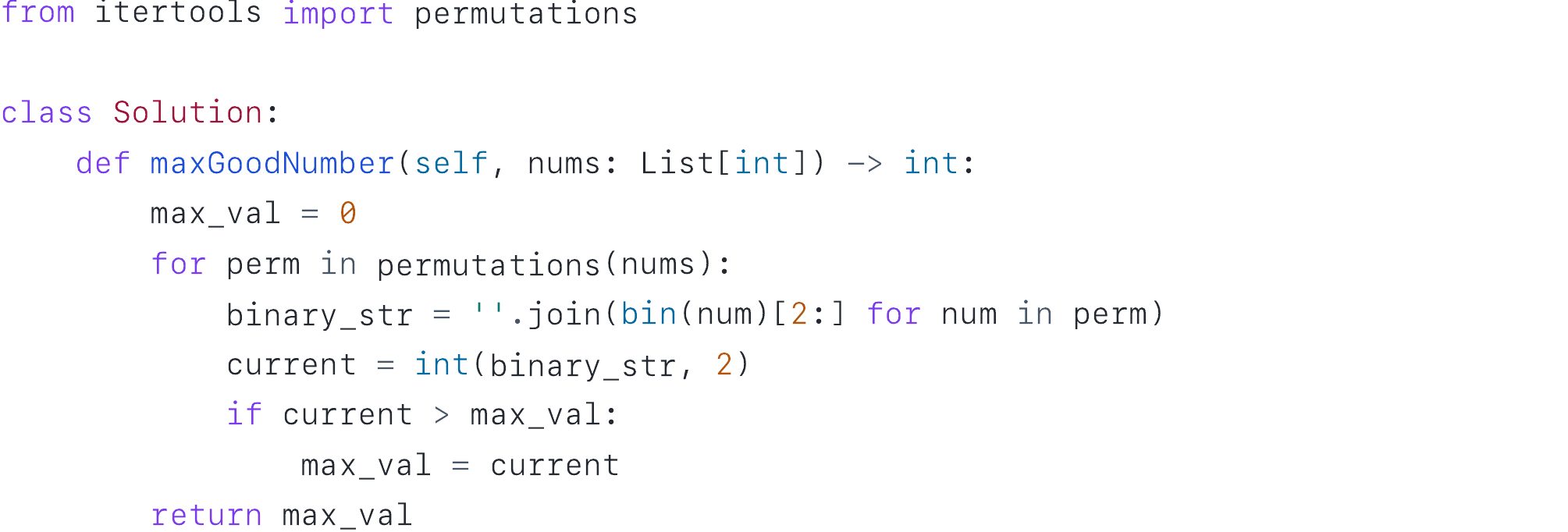}
\end{tcolorbox}

\begin{tcolorbox}[title=SFT-as-Context Response, colback=gray!5, colframe=gray, fonttitle=\bfseries, breakable]
\includegraphics[width=\linewidth]{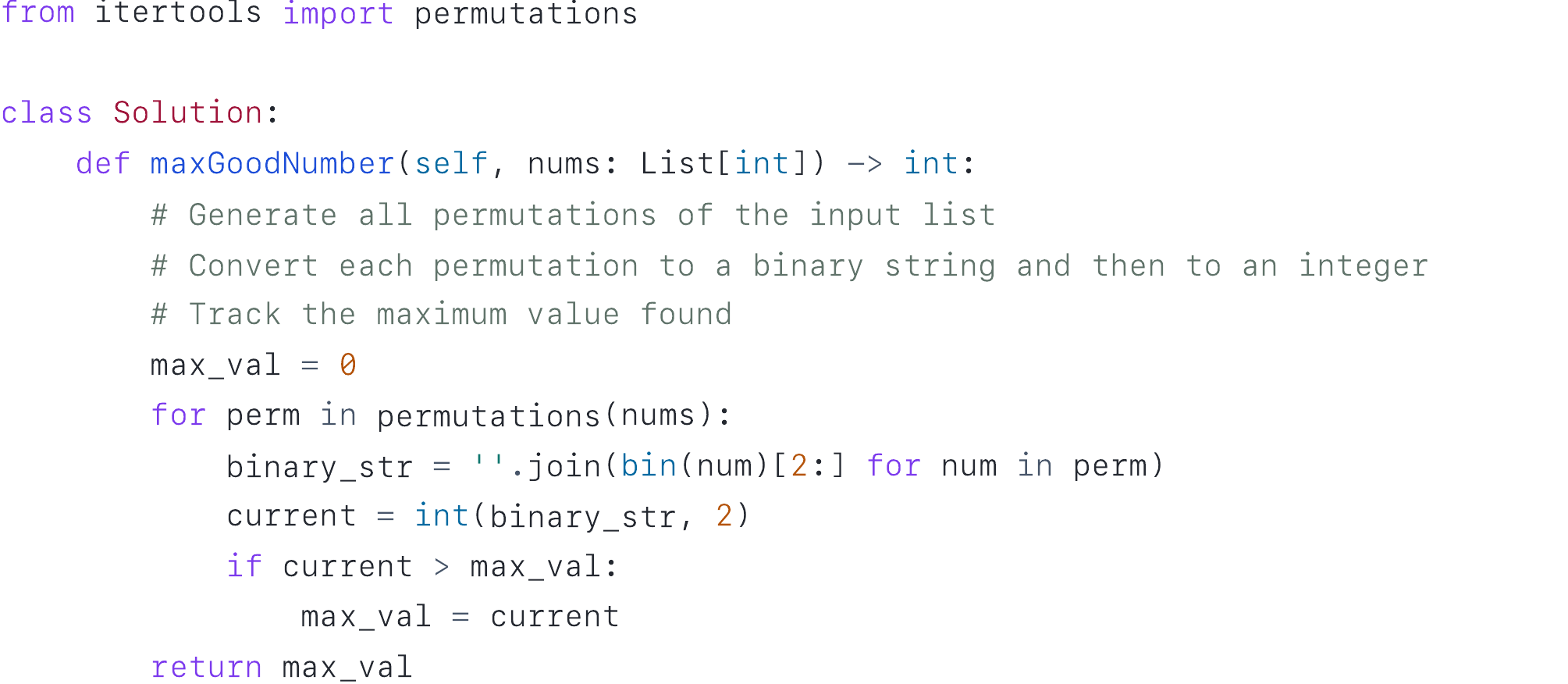}
\end{tcolorbox}

\subsection{SFT-as-Context Failure Mode: Boxed Answer Only}
\label{sec:sft-as-context-failure-mode}

We provide two examples where Qwen2.5-7B-Instruct fails to follow the original task instructions during SFT-as-context and only outputs a boxed answer.

In the first example below, only outputting the boxed answer violates the requirement of using the word ``number.''

\begin{tcolorbox}[title=Example Query 1, colback=gray!5, colframe=gray, fonttitle=\bfseries, breakable]
Zeke's baseball team has 7 more players than Carlton's. If Carlton's team has 13 players, how many players are there in both teams combined? Let's think step by step and output the final answer after ``\#\#\#\#''. Let's think step by step and output the final answer within \textbackslash boxed\{\}. In your response, the word ``number'' should appear at least 1 times. In your entire response, refrain from the use of any commas.
\end{tcolorbox}

\begin{tcolorbox}[title=Parent Response, colback=gray!5, colframe=gray, fonttitle=\bfseries, breakable]
Carlton's team number 13 players; Zeke's has 7 more making it 13+7=20 players; Combined both teams have 13+20=33 players \#\#\#\# \boxed{33}
\end{tcolorbox}

\begin{tcolorbox}[title=SFT Response, colback=gray!5, colframe=gray, fonttitle=\bfseries, breakable]
First find Zeke's team's number of players:13+7=20. Then add both teams:20+13=33.\#\#\#\# \boxed{33}\#\#\#\#
\end{tcolorbox}

\begin{tcolorbox}[title=SFT-as-Context Response, colback=gray!5, colframe=gray, fonttitle=\bfseries, breakable]
\boxed{33}\#\#\#\#
\end{tcolorbox}

In the second example below, only outputting the boxed answer violates the requirement of highlighting at least two sections.

\begin{tcolorbox}[title=Example Query 2, colback=gray!5, colframe=gray, fonttitle=\bfseries, breakable]
A pet store currently has 5 dogs, 2 cats, and 10 birds. How many legs in total do the pets in the store have? Let's think step by step and output the final answer after ``\#\#\#\#''. Let's think step by step and output the final answer within \textbackslash boxed\{\}. Highlight at least 2 sections in your answer with markdown, i.e. *highlighted section*.
\end{tcolorbox}

\begin{tcolorbox}[title=Parent Response, colback=gray!5, colframe=gray, fonttitle=\bfseries, breakable]
Sure, let's break this down step by step.\\

First, we need to determine the number of legs each type of pet has:\\
- Dogs have 4 legs each.\\
- Cats have 4 legs each.\\
- Birds have 2 legs each.\\

Next, we calculate the total number of legs for each type of pet:\\
- *For the dogs*: There are 5 dogs, and each dog has 4 legs. Therefore, the total number of legs for the dogs is \(5 \times 4 = 20\).\\
- *For the cats*: There are 2 cats, and each cat has 4 legs. Therefore, the total number of legs for the cats is \(2 \times 4 = 8\).\\
- *For the birds*: There are 10 birds, and each bird has 2 legs. Therefore, the total number of legs for the birds is \(10 \times 2 = 20\).\\

Finally, we sum the total number of legs from all the pets:\\
- The total number of legs is \(20 + 8 + 20 = 48\).\\

Therefore, the total number of legs the pets in the store have is \boxed{48}.
\end{tcolorbox}

\begin{tcolorbox}[title=SFT Response, colback=gray!5, colframe=gray, fonttitle=\bfseries, breakable]
The pet store has three types of pets: dogs, cats, and birds.  \\

**1. Calculating legs for dogs:**  \\
There are 5 dogs, each with 4 legs. 
\[ 5 \text{ dogs} \times 4 \text{ legs/dog} = 20 \text{ legs} \] 

**2. Calculating legs for cats:**  \\
There are 2 cats, each with 4 legs. 
\[ 2 \text{ cats} \times 4 \text{ legs/cat} = 8 \text{ legs} \] 

**3. Calculating legs for birds:**  \\
There are 10 birds, each with 2 legs.  
\[ 10 \text{ birds} \times 2 \text{ legs/bird} = 20 \text{ legs} \]  

**Total legs:**  \\
\[ 20 \text{ (dogs)} + 8 \text{ (cats)} + 20 \text{ (birds)} = 48 \text{ legs} \]  

**Final Answer:**  \\
\boxed{48}
\end{tcolorbox}

\begin{tcolorbox}[title=SFT-as-Context Response, colback=gray!5, colframe=gray, fonttitle=\bfseries, breakable]
\#\#\#\#\\

\boxed{48}
\end{tcolorbox}

\section{Attention Visualization}
\label{sec:attention-visualization}

To investigate how strongly the model attends to the provided context, we analyze the attention weights assigned to prompt tokens during response generation. For each response token, we extract its attention weights over all prompt tokens (which we refer to as ``prompt-directed attention''), excluding attention to preceding response tokens. We then normalize these weights over the prompt tokens so that they sum to 100\%, with each prompt token's share proportional to its attention weight. Finally, we average the normalized prompt-directed attention over response tokens, attention heads, and transformer layers. Chat wrappers and special tokens are excluded from the calculation.

Formally, the share of prompt-directed attention is calculated as follows.

For each example $e$, the prompt is split into three regions: the query $Q$, the SFT-as-context instructions $I$, and the SFT context $C$. Let $A_{e,\ell,h,i,j}$ denote the attention weight assigned to prompt token $i$ by head $h$ in layer $\ell$, for generating response token $j$. Let $w_{e,i,R}\in\{0,1\}$ be a weight that indicates which region $R\in\{Q,I,C\}$ a prompt token $i$ belongs to. This weight is $1$ for a token inside $R$ and $0$ for a token outside $R$.

We first sum the attention assigned to each region:
\begin{equation*}
    m_{e,\ell,h,j}(R)
    = \sum_{i} w_{e,i,R}\,A_{e,\ell,h,i,j},
    \qquad R\in\{Q,I,C\}.
\end{equation*}
The total attention directed to the prompt is then
\begin{equation*}
    d_{e,\ell,h,j}
    = m_{e,\ell,h,j}(Q)
    + m_{e,\ell,h,j}(I)
    + m_{e,\ell,h,j}(C).
\end{equation*}
The share of the attention weight for each prompt region, for each layer, head, and response token $j$ is then:
\begin{equation*}
    s_{e,\ell,h,j}(R)
    = \frac{m_{e,\ell,h,j}(R)}{d_{e,\ell,h,j}}.
\end{equation*}

Next, let $\mathcal{L}$ be the included attention layers, $H$ the number of
attention heads per layer, and $J$ the set of response tokens. The percentage share for each region $R$ in one example $e$ is
\begin{equation*}
    S_e(R)
    = \frac{100}{|\mathcal{L}|\,H\,|J|}
      \sum_{\ell\in\mathcal{L}}
      \sum_{h=1}^{H}
      \sum_{j\in J}
      s_{e,\ell,h,j}(R).
\end{equation*}
We use all generated non-special tokens, all eight
heads ($H=8$), and the seven full-attention layers
$\mathcal{L}=\{6,12,18,24,30,36,42\}$ in Gemma-4-E4B-it.

Finally, for a category containing $N$ examples, we average the per-example
percentages to obtain the final share of prompt-directed attention for each region:
\begin{equation*}
    S_{\mathrm{category}}(R)
    = \frac{1}{N}\sum_{e=1}^{N}S_e(R).
\end{equation*}
Here, the category is relevant or irrelevant, and the number of queries we sample for each category is $N=100$.

\Cref{fig:attention-comparison-token-level} shows the token-level distribution of prompt-directed attention. For token-level visualization, we treat each prompt token as an individual region, while keeping the calculation the same. The parent model places greater attention on tokens in the original instruction specifying that non-food queries should be flagged as irrelevant. This observation provides a possible explanation for why SFT-as-context recovers general instruction-following capabilities even when the context is generated by an SFT model that has lost these capabilities.

\begin{figure}[t]
    \centering
    \includegraphics[width=\linewidth]{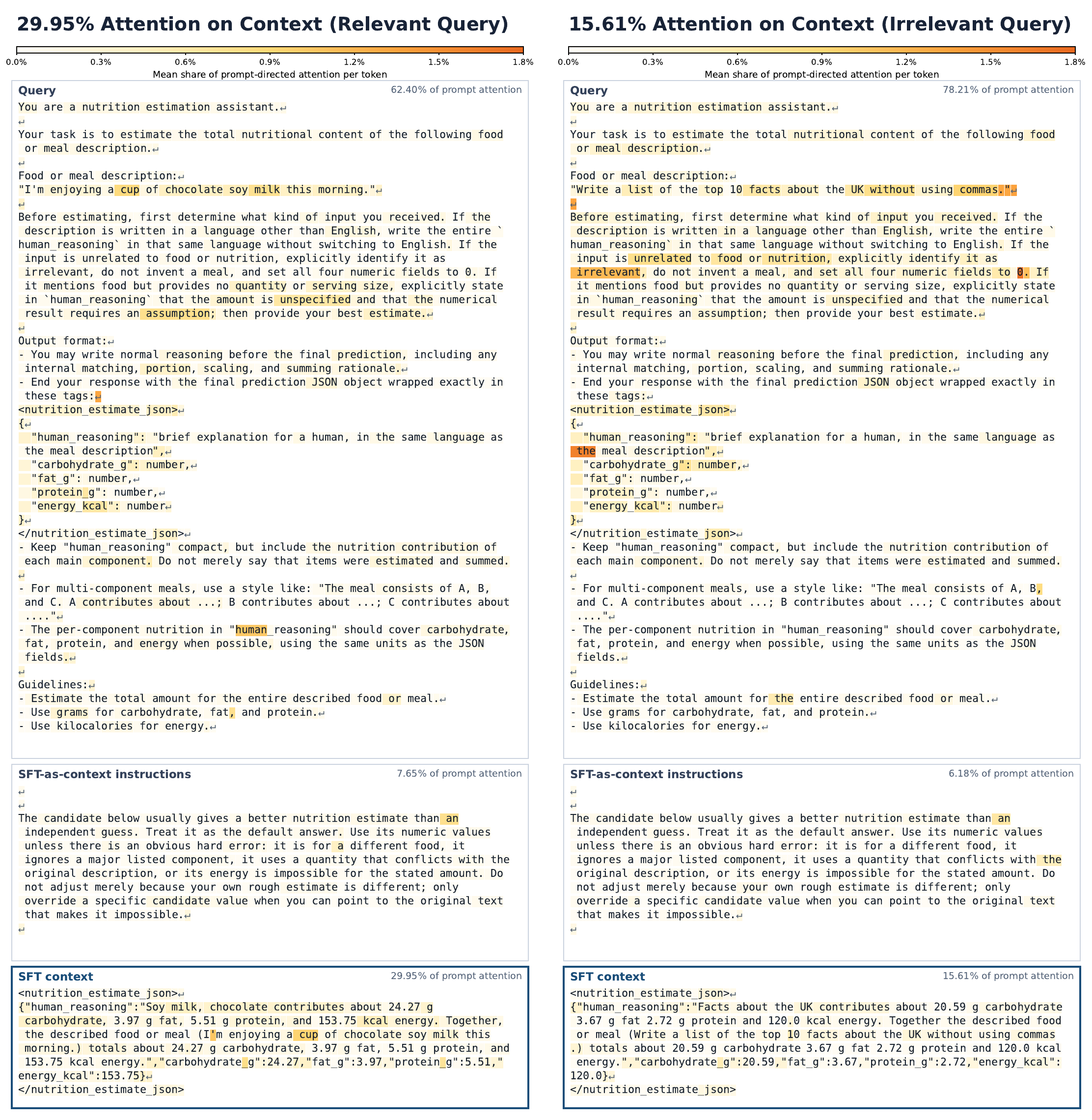}
    \caption{\textbf{The parent model attends less to the SFT context when the context is hallucinated for non-food queries.} These two examples show that the mean share of prompt-directed attention is higher for useful context (29.95\%) than for hallucinated context (15.61\%). Words like ``unrelated'' and ``irrelevant'' are more attended to when the query does not contain a food, demonstrating that the parent model is not distracted by hallucinated context and adheres to the original task instructions.}
    \label{fig:attention-comparison-token-level}
\end{figure}

\section{RL Models Forget Less}
\label{sec:rl-models-forget-less}

We have evaluated SFT-as-context in a controlled setting where the model providing the context is obtained through standard SFT. Since SFT-as-context is agnostic to the post-training method used to obtain the context-providing model, we would like to further investigate whether it generalizes to model pairs produced using alternative post-training methods. Specifically, we consider the model pairs in \Cref{tab:rl-results}, where the parent models belong to the DeepSeek-R1-Distill-Qwen family~\citep{guo2025deepseek} and the child models are further trained with reinforcement learning (RL). However, consistent with prior work~\citep{chen2026retaining}, we find that RL post-training induces negligible forgetting of instruction-following capabilities, as measured by IFEval. Hence, we do not apply SFT-as-context (or RL-as-context in this case) to close the gaps that are already small.

\begin{table*}[t]
  \centering
  \caption{\textbf{RL-trained models do not show forgetting of instruction-following capability compared to their parent, as measured by IFEval.} All RL models in this table share DeepSeek-R1-Distill-Qwen of the corresponding size (1.5B or 7B) as their parent model. The metric is prompt-level strict accuracy (PSA) of IFEval. $\Delta$ is RL minus parent in percentage points.}
  \label{tab:rl-results}
  \small
  \setlength{\tabcolsep}{6pt}
  \renewcommand{\arraystretch}{1.05}
  \begin{tabular*}{\textwidth}{@{\extracolsep{\fill}}lccc@{}}
    \toprule
    RL Model & Parent PSA & RL PSA & $\Delta$ (pp) \\
    \midrule
    DeepScaleR-1.5B-Preview~\citep{deepscaler2025} & 36.78 & 39.00 & $+2.22$ \\
    FastCuRL-1.5B-Preview~\citep{song-etal-2025-fastcurl} & 36.78 & 33.27 & $-3.51$ \\
    FastCuRL-1.5B-V2~\citep{song-etal-2025-fastcurl} & 36.78 & 36.04 & $-0.74$ \\
    FastCuRL-1.5B-V3~\citep{song-etal-2025-fastcurl} & 36.78 & 35.86 & $-0.92$ \\
    DeepMath-1.5B~\citep{he2026deepmath} & 36.78 & 37.52 & $+0.74$ \\
    AReaL-1.5B-Preview-Stage-3~\citep{fu2025areal} & 36.78 & 37.34 & $+0.55$ \\
    STILL-3-1.5B-preview~\citep{chen2025empirical} & 36.78 & 36.78 & $0.00$ \\
    RLinf-math-1.5B~\citep{yu2026rlinf} & 36.78 & 43.25 & $+6.47$ \\
    AceMath-RL-Nemotron-7B~\citep{liu-etal-2025-acemath} & 58.78 & 56.75 & $-2.03$ \\
    Skywork-OR1-Math-7B~\citep{he2025skywork} & 58.78 & 55.64 & $-3.14$ \\
    Polaris-7B-Preview~\citep{Polaris2025} & 58.78 & 56.75 & $-2.03$ \\
    \midrule
    \textbf{Average (11 pairs)} & 42.78 & 42.56 & $-0.22$ \\
    \bottomrule
  \end{tabular*}
\end{table*}

\section{A Case Study on Safety}
\label{sec:a-case-study-on-safety}

An LLM could be adversarially fine-tuned to remove safety guardrails. Here, the fine-tuned capability directly conflicts with its general capabilities. This setting differs from the other experiments in this paper, where fine-tuned and general capabilities are orthogonal.

To examine whether SFT-as-context favors fine-tuned or general capabilities under such a conflict, we use the widely used series of jailbroken models released by Huihui AI on Hugging Face.\footnote{\url{https://huggingface.co/huihui-ai/models}} Although these models use an abliteration process that is not SFT~\citep{arditi2024refusal}, we nevertheless simulate a setting where an SFT model provides harmful responses. We test the parent and abliterated models on HarmBench~\citep{harmbench}. \Cref{tab:safety-results} shows that SFT-as-context largely preserves the harmful responses produced by the abliterated model. In other words, when fine-tuned and general capabilities conflict, the parent model follows the fine-tuned capability under SFT-as-context. One possible explanation is that our SFT-as-context instruction does not explicitly impose safety requirements. Consequently, the parent model may focus on verifying whether the response satisfies the task requirements without first determining whether those requirements are safe to fulfill. Future work could investigate how SFT-as-context instructions can be optimized to prioritize the recovery of specific general capabilities.

\begin{table*}[t]
  \centering
  \caption{\textbf{SFT-as-context fails to recover the safety capability of the parent model when the context is adversarial.} Abliterated models show far higher attack success rates (ASR) on HarmBench than their parent models. SFT-as-context largely preserves the harmfulness in the context, with an average recovery rate of only 16.44\%. The average recovery rate is calculated from the averaged scores.}
  \label{tab:safety-results}
  \small
  \setlength{\tabcolsep}{6pt}
  \renewcommand{\arraystretch}{1.05}
  \begin{tabular*}{\textwidth}{@{\extracolsep{\fill}}lcccc@{}}
    \toprule
    Parent Model & Parent & Abliterated & SFT-as-Context & Recovery \\
    \midrule
    Qwen2.5-7B-Instruct
      & \textbf{19.00} & 63.25 & \underline{60.50} & 6.21 \\
    Qwen2.5-32B-Instruct
      & \textbf{14.50} & 72.50 & \underline{62.00} & 18.10 \\
    Qwen3-4B
      & \textbf{28.00} & 66.25 & \underline{65.75} & 1.31 \\
    Qwen3-8B
      & \textbf{24.25} & 71.00 & \underline{67.00} & 8.56 \\
    Gemma-4-E4B-it
      & \textbf{9.50} & 72.75 & \underline{71.75} & 1.58 \\
    Gemma-4-12B-it
      & \textbf{8.75} & 70.75 & \underline{46.75} & 38.71 \\
    Gemma-4-31B-it
      & \textbf{18.00} & 81.00 & \underline{62.00} & 30.16 \\
    \midrule
    \textbf{Average (7 pairs)}
      & \textbf{17.43} & 71.07 & \underline{62.25} & 16.44 \\
    \bottomrule
  \end{tabular*}
\end{table*}

\section{Forgetting Cannot Be Mitigated by Oracle Routing}

Given a pair of parent and SFT models, queries can be routed at inference time. Queries requiring fine-tuned capabilities can be answered by the SFT model, while those requiring general capabilities can be answered by the parent model. However, real-world queries often require both fine-tuned and general capabilities simultaneously. In such cases, routing between the two models may be insufficient, since each query is ultimately answered by only one model.

To demonstrate this limitation, we compare SFT-as-context with an oracle router that selects the better of the responses generated by the parent and SFT models. We evaluate performance using the joint success rate, which requires a response to satisfy both the fine-tuned and general capabilities. For MathIF, joint success requires both a correct answer and compliance with all constraints (hard instruction-following). For LiveCodeBenchIF, it requires both Pass@1 and exactly three comment lines. As shown in Tables~\ref{tab:mathif-conjunction-results} and~\ref{tab:livecodebenchif-conjunction-results}, SFT-as-context outperforms even the oracle router. These results highlight a fundamental limitation of routing. Even with oracle selection between the parent and SFT models, neither model alone can reliably satisfy requirements that depend on both sets of capabilities. In contrast, SFT-as-context enables the two capabilities to be combined within a single response.

\begin{table*}[t]
  \centering
  \caption{\textbf{MathIF joint-success rates show that SFT-as-context outperforms the oracle router.} A response counts as a joint success when the answer is correct and Hard IF is satisfied. Improvement is defined as SFT-as-context performance minus oracle performance. Best results are bolded and second-best results are underlined.}
  \label{tab:mathif-conjunction-results}
  \small
  \setlength{\tabcolsep}{6pt}
  \begin{tabular*}{\textwidth}{@{\extracolsep{\fill}}lccccc@{}}
    \toprule
    SFT Model & Parent & SFT & Oracle & SFT-as-Context & Improvement \\
    \midrule
    OpenThinker3-1.5B
      & 3.6 & 6.7 & \underline{8.3} & \textbf{10.2} & $+1.9$ \\
    OpenThinker2-7B
      & 17.4 & 13.3 & \underline{22.6} & \textbf{28.3} & $+5.7$ \\
    OpenThinker3-7B
      & 17.4 & 13.6 & \underline{23.3} & \textbf{30.0} & $+6.7$ \\
    Bespoke-Stratos-7B
      & \underline{17.4} & 13.6 & \textbf{22.4} & \textbf{22.4} & $0.0$ \\
    DistilQwen-ThoughtX-7B
      & 17.4 & 13.1 & \underline{22.4} & \textbf{27.6} & $+5.2$ \\
    s1.1-32B
      & 31.9 & 17.6 & \underline{38.6} & \textbf{41.4} & $+2.9$ \\
    LIMO
      & 31.9 & 22.9 & \underline{41.0} & \textbf{43.8} & $+2.9$ \\
    OpenThinker2-32B
      & 31.9 & 20.2 & \underline{38.8} & \textbf{41.7} & $+2.9$ \\
    \midrule
    \textbf{Average (8 pairs)}
      & 21.1 & 15.1 & \underline{27.2} & \textbf{30.7} & $+3.5$ \\
    \bottomrule
  \end{tabular*}
\end{table*}

\begin{table*}[t]
  \centering
  \caption{\textbf{LiveCodeBenchIF joint-success rates show that SFT-as-context outperforms the oracle router.} A response counts as a joint success when the answer is correct and the number of comments is exactly three. Improvement is defined as SFT-as-context performance minus oracle performance. The only exception where the improvement is negative is OlympicCoder-7B. Best results are bolded and second-best results are underlined. }
  \label{tab:livecodebenchif-conjunction-results}
  \small
  \setlength{\tabcolsep}{6pt}
  \begin{tabular*}{\textwidth}{@{\extracolsep{\fill}}lccccc@{}}
    \toprule
    SFT Model & Parent & SFT & Oracle & SFT-as-Context & Improvement \\
    \midrule
    OlympicCoder-7B
      & 3.4 & \underline{9.3} & \textbf{11.2} & \underline{9.3} & $-1.9$ \\
    OpenCodeReasoning-Nemotron-7B
      & 2.2 & 7.1 & \underline{8.6} & \textbf{10.8} & $+2.2$ \\
    OpenCodeReasoning-Nemotron-1.1-7B
      & 2.2 & 9.0 & \underline{10.1} & \textbf{10.8} & $+0.7$ \\
    OpenCodeReasoning-Nemotron-14B
      & 5.2 & 16.8 & \underline{20.1} & \textbf{25.7} & $+5.6$ \\
    OpenCodeReasoning-Nemotron-1.1-14B
      & 5.2 & 17.2 & \underline{19.8} & \textbf{26.1} & $+6.3$ \\
    OlympicCoder-32B
      & 13.4 & \underline{16.0} & \textbf{23.9} & \textbf{23.9} & $0.0$ \\
    OpenCodeReasoning-Nemotron-32B
      & 9.7 & 21.6 & \underline{26.1} & \textbf{27.2} & $+1.1$ \\
    OpenCodeReasoning-Nemotron-1.1-32B
      & 9.7 & 24.6 & \underline{29.9} & \textbf{31.7} & $+1.9$ \\
    \midrule
    \textbf{Average (8 pairs)}
      & 6.4 & 15.2 & \underline{18.7} & \textbf{20.7} & $+2.0$ \\
    \bottomrule
  \end{tabular*}
\end{table*}

\section{Token-Level Manipulation Fails to Mitigate Forgetting for NutriBench}
\label{sec:token-level-manipulation}

Concurrent work has shown that token-level manipulation methods, such as ensembling and contrastive decoding, can mitigate forgetting on certain benchmarks~\citep{ki2026cpr}. However, we identify a failure case for these methods. Following the implementation of \citet{ki2026cpr}, we use ensembling and contrastive decoding as two strong baselines. For ensembling, we average the token probabilities of the parent and SFT models with equal weights. For contrastive decoding, we treat the SFT model as the expert model and the parent model as the amateur model~\citep{li-etal-2023-contrastive}. From NutriBench, we sample 100 English queries, 100 non-English queries, and 100 non-food queries, and compare these baseline methods with SFT-as-context across all three nutrition model pairs.

\Cref{tab:nutribench-token-level-baselines} shows that, although both token-level baselines sometimes outperform SFT-as-context on fine-tuned capabilities, neither successfully mitigates forgetting of general capabilities. This finding contrasts with prior results on other benchmarks~\citep{ki2026cpr}, where both baselines can often fully mitigate forgetting. We attribute this failure to the specific failure patterns exhibited by the SFT models. For non-English queries, the SFT model has a strong tendency to generate the English word ``contributes'' after the name of a food item. The disproportionately high probability assigned to this token makes it more likely to be selected under both ensembling and contrastive decoding, resulting in poor language consistency. Similarly, for non-food queries, ``contributes'' tends to dominate the token distribution, leading the model to continue to generate hallucinated, non-zero nutritional values. These results suggest that SFT-as-context can address certain shifts in token probability distributions that are difficult to correct through token-level manipulation.

\begin{table*}[t]
  \centering
  \caption{\textbf{Token-level baselines (ensemble and contrastive decoding) fail to recover general capabilities.} The results for each inference setting are based on 100 English, 100 non-English, and 100 non-food queries randomly sampled from NutriBench. Best results are bolded and second-best results are underlined.}
  \label{tab:nutribench-token-level-baselines}
  \small
  \setlength{\tabcolsep}{4pt}
  \begin{tabular*}{\textwidth}{@{\extracolsep{\fill}}lcccc@{}}
    \toprule
    & \multicolumn{2}{c}{Fine-Tuned Capabilities}
    & \multicolumn{2}{c}{General Capabilities} \\
    \cmidrule(lr){2-3}\cmidrule(lr){4-5}
    Inference Setting
      & \shortstack{English Macro\\MAE $\downarrow$}
      & \shortstack{Non-English Macro\\MAE $\downarrow$}
      & \shortstack{Language\\Consistency (\%) $\uparrow$}
      & \shortstack{Acknowledge\\Non-Food (\%) $\uparrow$} \\
    \midrule

    \rowcolor{black!10}[0pt][0pt]
    \multicolumn{5}{@{}c@{}}{\textbf{Gemma-4-E4B-it}} \\
    \addlinespace[2pt]
    Parent
      & 35.6 & 75.4 & \textbf{96} & \underline{97} \\
    SFT
      & \textbf{14.9} & \textbf{54.3} & 4 & 37 \\
    Ensemble
      & \underline{16.1} & 64.4 & \underline{11} & 44 \\
    Contrastive Decoding
      & 17.3 & 63.0 & 1 & 24 \\
    SFT-as-Context
      & 16.4 & \underline{60.6} & \textbf{96} & \textbf{100} \\

    \midrule
    \rowcolor{black!10}[0pt][0pt]
    \multicolumn{5}{@{}c@{}}{\textbf{Qwen3-4B}} \\
    \addlinespace[2pt]
    Parent
      & 39.5 & 99.7 & \textbf{74} & \textbf{99} \\
    SFT
      & \textbf{15.5} & \underline{79.0} & 0 & 18 \\
    Ensemble
      & 56.5 & 155.7 & 23 & 26 \\
    Contrastive Decoding
      & \underline{17.5} & \textbf{57.1} & 0 & 8 \\
    SFT-as-Context
      & 19.4 & 80.3 & \underline{66} & \underline{97} \\

    \midrule
    \rowcolor{black!10}[0pt][0pt]
    \multicolumn{5}{@{}c@{}}{\textbf{Qwen3-8B}} \\
    \addlinespace[2pt]
    Parent
      & 35.7 & 81.5 & \underline{66} & \textbf{100} \\
    SFT
      & \textbf{15.6} & \underline{57.1} & 6 & 8 \\
    Ensemble
      & 28.4 & 92.1 & 36 & 11 \\
    Contrastive Decoding
      & \underline{17.5} & 60.7 & 0 & 6 \\
    SFT-as-Context
      & 20.0 & \textbf{54.2} & \textbf{67} & \underline{96} \\
    \bottomrule
  \end{tabular*}
\end{table*}

\section{SFT-as-Context Incurs Modest Overhead Compared to SFT}
\label{sec:sft-as-context-incurs-modest-overhead-compared-to-SFT}

Since SFT-as-context requires two inference passes, we examine the additional generation overhead introduced by the second pass, averaged across all responses for each benchmark. \Cref{tab:response-token-ratio} shows that, for all math and coding model pairs, the second pass incurs only modest overhead relative to the typically long responses generated by the SFT model in the first pass. Specifically, the ratio of second-pass to first-pass response tokens is at most 13.0\%.

For the nutrition estimation task, however, the parent model generates relatively longer responses during the second pass, as it performs additional reasoning to compare the SFT-generated reference with its own estimate. Nevertheless, the absolute number of generated tokens remains comparable to that of the other tasks, ranging from a few hundred to slightly over one thousand tokens. Overall, these results suggest that the additional generation overhead introduced by SFT-as-context is modest.

\begin{table*}[t]
  \centering
  \caption{\textbf{SFT-as-context incurs modest overhead compared with SFT.} Under each of the 3 domains, we list benchmarks for testing both fine-tuned and general capabilities. Ratio is the number of SFT-as-context response tokens divided by the number of SFT response tokens, reported as percentages. Although SFT-as-context adds a second inference pass after the first inference pass of the SFT model, only a small number of additional tokens is generated in the second pass for most benchmarks. Nutrition estimation is an outlier, where SFT-as-context generates more response tokens to compare the reference with its own estimation.}
  \label{tab:response-token-ratio}
  \small
  \setlength{\tabcolsep}{8pt}
  \begin{tabular*}{\textwidth}{@{\extracolsep{\fill}}lccc@{}}
    \toprule
    Benchmark
      & \shortstack{SFT Response Tokens}
      & \shortstack{SFT-as-Context Response Tokens}
      & Ratio (\%) \\
    \midrule

    \rowcolor{black!10}[0pt][0pt]
    \multicolumn{4}{@{}c@{}}{\textbf{Mathematics}} \\
    \addlinespace[2pt]
    AIME 2024
      & 13,814 & 698 & 5.1 \\
    MathIF
      & 6,833 & 342 & 5.0 \\
    IFEval
      & 7,469 & 972 & 13.0 \\
    MGSM
      & 5,851 & 396 & 6.8 \\

    \midrule
    \rowcolor{black!10}[0pt][0pt]
    \multicolumn{4}{@{}c@{}}{\textbf{Coding}} \\
    \addlinespace[2pt]
    LiveCodeBench
      & 14,347 & 350 & 2.4 \\
    LiveCodeBenchIF
      & 14,885 & 387 & 2.6 \\
    IFEval
      & 10,816 & 523 & 4.8 \\
    SQuAD2.0
      & 10,799 & 8 & 0.1 \\
    FQuAD2.0
      & 7,174 & 39 & 0.6 \\
    CoQA
      & 13,021 & 41 & 0.3 \\
    ACPBench
      & 10,618 & 196 & 1.8 \\

    \midrule
    \rowcolor{black!10}[0pt][0pt]
    \multicolumn{4}{@{}c@{}}{\textbf{Nutrition}} \\
    \addlinespace[2pt]
    NutriBench-English
      & 415 & 823 & 198.1 \\
    NutriBench-Non-English
      & 870 & 1,252 & 143.9 \\
    NutriBench-Non-Food
      & 1,000 & 300 & 30.0 \\
    \bottomrule
  \end{tabular*}
\end{table*}

\section{Additional Ablation Studies}
\label{sec:additional-ablations-studies}

\paragraph{Repeating task instruction helps SFT-as-context.} For math and coding models, we repeat the original task instruction in the SFT-as-context prompt. On MathIF, we perform an ablation study by removing the repeated instructions (from ``Again, the original task is'' to the end of the SFT-as-context instructions). \Cref{tab:mathif-repeat-ablation} shows that repeating the original task instructions significantly helps instruction-following, while barely affecting the fine-tuned math capability.

\begin{table*}[t]
  \centering
  \caption{\textbf{Repeating the original task prompt in the SFT-as-context instructions helps instruction-following.} For the benchmark MathIF, repeating the instruction significantly helps the instruction-following (Hard IF and Soft IF) accuracies of SFT-as-context. All results are percentages, averaged over 8 math model pairs. $\Delta$ is
  SFT-as-context without repeating original task minus SFT-as-context, in percentage
  points.}
  \label{tab:mathif-repeat-ablation}
  \small
  \setlength{\tabcolsep}{5pt}
  \begin{tabular*}{\textwidth}{@{\extracolsep{\fill}}lccccc@{}}
    \toprule
    MathIF Metric
      & Parent
      & SFT
      & SFT-as-Context
      & \shortstack{SFT-as-Context w/o Repeat}
      & $\Delta$ (pp) \\
    \midrule
    Math Correctness
      & 40.6 & 65.0 & 64.0 & 64.2 & $+0.2$ \\
    Hard IF
      & 50.1 & 21.3 & 45.8 & 36.1 & $-9.7$ \\
    Soft IF
      & 67.8 & 41.5 & 65.4 & 56.6 & $-8.8$ \\
    \bottomrule
  \end{tabular*}
\end{table*}

\section{Theoretical Results and Proofs}
\label{sec:theoretical-results-and-proofs}

Given the above model of sequential generation in Section \ref{sec:theoretical-explanation}, we characterize the errors for both in-domain and out-of-domain queries.
We first describe the optimal response distributions under the optimal domain variable using a variable: $\theta^* \in \Theta$, where $\Theta$ is discrete.
As above, the parent model is assumed to be more general-purpose, meaning it has a set of domain expertise $\Theta_{\parent}$.
Let the optimal output, denoted by $Y$, be generated given the prompt in some domain with the domain expertise vector $\theta^*$:
\begin{align}
\label{eq:optimal_reasoning}
    Y \sim \Pr(Y | x, \theta^*).
\end{align}

We denote by $\Pr_{*}$ the above output distribution under optimal domain variable.

\subsection{Additional Intuition for the Bayesian Framework}
\label{sec:bayesian-intuition}

The domain variable $\theta$ captures differences in the
query-response distributions across domains.
For example, nutrition estimation and mathematical reasoning
involve different response distributions, modeled through
$\Pr(Y\mid x,\theta)$.
A parent model trained across multiple domains can represent
several such capabilities; for example, Qwen2.5-1.5B-Instruct supports mathematical reasoning, multilingual
understanding, and code generation.
Fine-tuning on query-response pairs from a single domain
$\theta_{\sft}$ aims to improve the corresponding
domain-specific capability.

Under the Bayesian interpretation
\citep{xie2022an, hu2024unveiling}, the response
distribution is obtained by marginalizing over the domain
variable.
Conceptually, generation first selects a domain according
to its posterior probability and then generates a response
from the corresponding conditional output distribution.
This is a probabilistic interpretation rather than a claim
that the model explicitly performs these two steps.
SFT-as-context supplies the SFT response $y_{\sft}$ and
instructions $I$ as additional evidence for domain inference.
Under our shared conditional-output model, this evidence
changes the domain posterior while leaving the conditional
response distributions fixed.
Thus, the modeled improvement arises from assigning greater
probability to an existing domain capability.

\subsection{Defining the Errors}

\paragraph{Error of Parent Model:} The parent model has plenty of domains ($\Theta_{\parent}$) in its set of capabilities and it need not always generate output dependent on the optimal domain $D$.
Its own output is therefore:
\begin{align*}
    Y_{\parent} \sim \Pr_{\parent}(Y  | x),
\end{align*}

and therefore the output distribution is obtained by marginalizing over all the domains:

\begin{align*}
    \Pr_{\parent}(y | x) = \sum_{\theta \in \Theta_{\parent}} \Pr_{\parent}(y | \theta, x ) \Pr_{\parent}(\theta | x).
\end{align*}

Therefore, the error for the parent model is the KL divergence between the optimal output distribution under the optimal reasoning law and the parent distribution:
\begin{align*}
    \err_{\parent} = KL( \Pr_*(y | x) || \Pr_{\parent}(y | x)).
\end{align*}

\paragraph{Error of SFT Model:} We assume that the SFT model has expertise in one domain specifically $D_{\sft}$ with domain variable $\theta_{\sft}$.
As discussed previously, due to forgetting, the model has a reduced set of domains, denoted by $\Theta_{\sft}$.
Therefore, the output of the SFT model is:
\begin{align*}
    \Pr_{\sft} (y | x) = \sum_{\theta \in \Theta_{\sft}} \Pr_{\sft}(y | \theta, x) \Pr_{\sft}(\theta | x).
\end{align*}

The associated error for the SFT model  is as follows:

\begin{align*}
    \err_{\sft} = KL (\Pr_* (y | x) || \Pr_{\sft}(y | x) ).
\end{align*}

\paragraph{Error of SFT-as-Context:} Similarly, we define the latent trajectory under SFT-as-context (abbreviated as SAC; similar to Equation \ref{eq:optimal_reasoning}).
\begin{align*}
    Y_{\sftc} \sim \Pr_{\sftc} (Y | x, \prompt, y_{\sft}).
\end{align*}

The output distribution is therefore obtained by marginalizing over all domains $D$ under the SAC distribution:

\begin{align*}
    \Pr_{\sftc} (y | x, \prompt, y_{\sft}) = \sum_{\theta \in \Theta_{\parent}} \Pr_{\parent}(y | \theta, x) \Pr_{\sftc}(\theta | x, \prompt, y_{\sft}).
\end{align*}

Therefore, the error here is:
\begin{align*}
    \err_{\sftc} = KL (\Pr_{*}(y| x) ||  \Pr_{\sftc} (y | x, \prompt, y_{\sft})).
\end{align*}

Using these error formulations, we split Theorem \ref{thm:sftc_guarantees} into two theorems and their corresponding proofs, and finally present Theorem \ref{thm:out-of-domain} and its proof.
Each following subsection corresponds to one theorem and its proof.

\subsection{SFT-as-Context In-Domain Error}
\label{sec:sft_in_domain_proof}

Suppose the SFT model is fine-tuned on data from the SFT domain with optimal domain variable $\theta_{\sft}$.
For an in-domain query $x \in D_{\sft}$, assume that
$\theta_{\sft} \in \Theta_{\parent}$ and that the optimal
output law is realizable:
\begin{align}
\label{eq:indomain_realizability}
    \Pr_{\parent}(y \mid \theta_{\sft},x)
    = \Pr_*(y \mid x).
\end{align}
Thus, the parent model can represent the optimal output law,
but may assign little probability to its domain.
All logarithms below are natural.

\begin{assumption}[Domain posterior bounds]
\label{ass:domain_concentration}
There exist $0 < \epsilon < 1$ and $0 \leq \epsilon' < 1$
such that, for each query and realized context under consideration,
\begin{align}
    0 < a_x
    &:= \Pr_{\parent}(\theta_{\sft} \mid x)
    \leq \epsilon, \\
    b_x
    &:= \Pr_{\sftc}(\theta_{\sft}
        \mid x,\prompt,y_{\sft})
    \geq 1-\epsilon'.
\end{align}
\end{assumption}

This assumption corresponds to Assumptions 1 and 2 in Section~\ref{sec:theoretical-explanation}, which we present together formally for ease of exposition in the proof below.

\begin{assumption}[Output separability]
\label{ass:output_separability}
Define the parent model's normalized output mixture over
the remaining domains by
\begin{align}
    Q_{\parent}(y \mid x)
    :=
    \frac{
        \sum_{\substack{\theta \in \Theta_{\parent}\\
                        \theta \neq \theta_{\sft}}}
        \Pr_{\parent}(y \mid \theta,x)
        \Pr_{\parent}(\theta \mid x)
    }{1-a_x}.
\end{align}
There exists $\epsilon'' > 0$ such that, for every
$x \in D_{\sft}$,
\begin{align}
    \operatorname{TV}\!\left(
        \Pr_*(\cdot \mid x),Q_{\parent}(\cdot \mid x)
    \right)
    \geq \epsilon'',
\end{align}
where $\operatorname{TV}(P,Q)
:= \frac12\sum_y |P(y)-Q(y)|$.
\end{assumption}

This assumption is Assumption 3 in Section \ref{sec:theoretical-explanation}. Separability is required for the mixture over the remaining
domains \citep{xie2022an}.
This gives us our first result.

\begin{theorem}[In-domain error reduction]
\label{thm:indomain_error_reduction}
Under Equation~\ref{eq:indomain_realizability} and
Assumptions~\ref{ass:domain_concentration}--\ref{ass:output_separability},
for each query and context satisfying these conditions,
\begin{align}
\label{eq:indomain_error_reduction}
    \err_{\parent}-\err_{\sftc}
    \geq
    2(1-\epsilon)^2(\epsilon'')^2+\log(1-\epsilon').
\end{align}
In particular, $\err_{\sftc}<\err_{\parent}$ whenever
\begin{align*}
    2(1-\epsilon)^2(\epsilon'')^2
    > -\log(1-\epsilon').
\end{align*}
\end{theorem}

\begin{proof}
Fix $x$ and the realized context $(\prompt,y_{\sft})$.
By realizability and the definition of $Q_{\parent}$,
\begin{align*}
    \Pr_{\parent}(y \mid x)
    = a_x\Pr_*(y \mid x)
      +(1-a_x)Q_{\parent}(y \mid x).
\end{align*}
Consequently (using Assumption \ref{ass:output_separability}),
\begin{align*}
    \operatorname{TV}\!\left(
        \Pr_*(\cdot \mid x),\Pr_{\parent}(\cdot \mid x)
    \right)
    &=
    (1-a_x)\operatorname{TV}\!\left(
        \Pr_*(\cdot \mid x),Q_{\parent}(\cdot \mid x)
    \right) \\
    &\geq (1-\epsilon)\epsilon''.
\end{align*}
The standard Pinsker's inequality therefore gives
\begin{align}
\label{eq:lower_bounf_parent}
    \err_{\parent}
    \geq 2(1-\epsilon)^2(\epsilon'')^2.
\end{align}

For SFT-as-context, retaining the optimal-domain term in its
nonnegative mixture gives
\begin{align*}
    \Pr_{\sftc}(y \mid x,\prompt,y_{\sft})
    \geq b_x\Pr_*(y \mid x).
\end{align*}
Hence,
\begin{align}
\begin{split}
    \err_{\sftc}
    &=
    \mathbb{E}_{Y\sim\Pr_*(\cdot\mid x)}
    \left[
        \log\frac{\Pr_*(Y \mid x)}
        {\Pr_{\sftc}(Y \mid x,\prompt,y_{\sft})}
    \right] \\
    &\leq -\log b_x
    \leq -\log(1-\epsilon').
\end{split}
\label{eq:sac_loss_bound}
\end{align}
Subtracting this upper bound from the lower bound on
$\err_{\parent}$ (Equation \ref{eq:lower_bounf_parent}) gives us the result.
\end{proof}

\subsection{In-Domain Closeness of SFT-as-Context to SFT}

We recall that the error for the SFT model's response distribution is defined as:
\begin{align*}
    \err_{\sft}
    := D_{\mathrm{KL}}\!\left(
        \Pr_*(\cdot \mid x)
        \,\|\, \Pr_{\sft}(\cdot \mid x)
    \right).
\end{align*}

\begin{assumption}[In-domain SFT realizability and concentration]
\label{ass:sft_accuracy}
The optimal domain $\theta_{\sft}$ belongs to
$\Theta_{\sft}$, and
\begin{align}
    \Pr_{\sft}(y \mid \theta_{\sft},x)
    &= \Pr_*(y \mid x).
\end{align}
Furthermore, there exists $0 \leq \eta < 1$ such that,
for every $x \in D_{\sft}$,
\begin{align}
    c_x
    := \Pr_{\sft}(\theta_{\sft} \mid x)
    \geq 1-\eta.
\end{align}
\end{assumption}

\begin{theorem}[Closeness to SFT]
\label{thm:sft_closeness}
Under the preceding SAC realizability and posterior
concentration assumptions, and
Assumption~\ref{ass:sft_accuracy}, we have
\begin{align}
    \log(1-\epsilon')
    \leq
    \err_{\sft}-\err_{\sftc}
    \leq
    -\log(1-\eta).
\end{align}
Consequently,
\begin{align}
    \left|\err_{\sft}-\err_{\sftc}\right|
    \leq
    \max\left\{
        -\log(1-\eta),
        -\log(1-\epsilon')
    \right\}.
\end{align}
\end{theorem}

\begin{proof}
Once again, fix an in-domain query $x \in D_{\sft}$ and a realized context
$(\prompt,y_{\sft})$ satisfying the assumptions.
Since the optimal domain $\theta_{\sft}$ is in the SFT model's capabilities $\Theta_{\sft}$, we have:
\begin{align*}
    \Pr_{\sft}(y \mid x)
    &\geq
    c_x\Pr_{\sft}(y \mid \theta_{\sft},x) \\
    &= c_x\Pr_*(y \mid x).
\end{align*}
By nonnegativity of KL divergence and this inequality,
\begin{align*}
    0
    \leq \err_{\sft}
    &=
    \mathbb{E}_{Y\sim\Pr_*(\cdot\mid x)}
    \left[
        \log\frac{\Pr_*(Y \mid x)}
        {\Pr_{\sft}(Y \mid x)}
    \right] \\
    &\leq -\log c_x
    \leq -\log(1-\eta).
\end{align*}
Similarly, the preceding SFT-as-context proof (Equation \ref{eq:sac_loss_bound}) gives us:
\begin{align*}
    0 \leq \err_{\sftc}
    \leq -\log(1-\epsilon').
\end{align*}
Combining these two intervals yields
\begin{align*}
    \log(1-\epsilon')
    \leq
    \err_{\sft}-\err_{\sftc}
    \leq
    -\log(1-\eta),
\end{align*}
which leads to the absolute-error bound.
\end{proof}

Unlike the improvement guarantee over the parent model,
this closeness guarantee does not require output separability.
Together, Theorems \ref{thm:indomain_error_reduction} and \ref{thm:sft_closeness} are presented as Theorem \ref{thm:sftc_guarantees} in Section \ref{sec:theoretical-explanation}.

\subsection{SFT-as-Context Out-of-Domain Error}
\label{sec:sft_out_of_domain_proof}

Consider a domain $D_o$ with domain variable
$\theta_o \in \Theta_{\parent}$, where
$\theta_o \neq \theta_{\sft}$.
For queries $x \in D_o$, assume that the optimal
output distribution is realizable in the parent model:
\begin{align}
\label{eq:ood_realizability}
    \Pr_{\parent}(y \mid \theta_o,x)
    &= \Pr_*(y \mid x), \\
    \Pr_{\parent}(\theta_o \mid x) &> 0.
\end{align}
In simpler terms, this means that the model has a non-zero probability of outputting the optimal response via the optimal task domain posterior probability distribution.
This is referred to as out-of-domain realizability.
Both $\err_{\parent}$ and $\err_{\sftc}$ are
measured relative to this optimal out-of-domain law.

\begin{assumption}[Out-of-domain posterior retention]
\label{ass:ood_retention}
There exists $0 \leq \rho_o < 1$ such that, for every
out-of-domain query and realized context under consideration,
\begin{align}
    \Pr_{\sftc}(\theta
        \mid x,\prompt,y_{\sft})
    \geq
    (1-\rho_o)\Pr_{\parent}(\theta \mid x),
    \qquad \forall \theta \in \Theta_{\parent}.
\end{align}
\end{assumption}

This assumption limits how much the SFT context can suppress
the parent model's existing domain probabilities and corresponds to Assumption 5 in Section \ref{sec:theoretical-explanation}.
It is an additional condition; membership of
$\theta_o$ in $\Theta_{\parent}$ alone need not imply it.

\begin{theorem}[Error guarantees for out-of-domain queries]
\label{thm:ood_error}
Under Equation~\ref{eq:ood_realizability},
Assumption~\ref{ass:ood_retention}, and the
conditional-output model described above,
\begin{align}
\label{eq:ood_error_bound}
    \err_{\sftc}-\err_{\parent}
    \leq -\log(1-\rho_o).
\end{align}
\end{theorem}

\begin{proof}
Fix $x \in D_o$ and a realized context
$(\prompt,y_{\sft})$ that satisfies the assumption.
Using posterior retention in the SAC mixture,
\begin{align*}
    &\Pr_{\sftc}(y \mid x,\prompt,y_{\sft}) \\
    &\quad =
    \sum_{\theta \in \Theta_{\parent}}
        \Pr_{\parent}(y \mid \theta,x)
        \Pr_{\sftc}(\theta
            \mid x,\prompt,y_{\sft}) \\
    &\quad \geq
    (1-\rho_o)
    \sum_{\theta \in \Theta_{\parent}}
        \Pr_{\parent}(y \mid \theta,x)
        \Pr_{\parent}(\theta \mid x) \\
    &\quad =
    (1-\rho_o)\Pr_{\parent}(y \mid x).
\end{align*}
The realizability assumption and the positive posterior of $\theta_o$
ensure that both errors are finite. Therefore,
\begin{align*}
    \err_{\sftc}
    &=
    \mathbb{E}_{Y\sim\Pr_*(\cdot\mid x)}
    \left[
        \log\frac{\Pr_*(Y \mid x)}
        {\Pr_{\sftc}(Y \mid x,\prompt,y_{\sft})}
    \right] \\
    &\leq
    \mathbb{E}_{Y\sim\Pr_*(\cdot\mid x)}
    \left[
        \log\frac{\Pr_*(Y \mid x)}
        {(1-\rho_o)\Pr_{\parent}(Y \mid x)}
    \right] \\
    &= \err_{\parent}-\log(1-\rho_o).
\end{align*}
Rearranging proves the result.
\end{proof}

Thus, when the SFT context only weakly suppresses the parent
posterior, the additional out-of-domain error is small:
$-\log(1-\rho_o)=\rho_o+O(\rho_o^2)$ as $\rho_o\to 0$.
This proves Theorem \ref{thm:out-of-domain} from Section \ref{sec:theoretical-explanation} above.

\end{document}